\documentclass{svmult}
 
\usepackage{mathptmx}  
\usepackage{helvet} 
\usepackage{courier}
\usepackage{amsmath,amsfonts}
\usepackage{graphicx}
\usepackage[bottom]{footmisc}          
    
\usepackage{mathptmx}         
\usepackage{helvet}             
\usepackage{courier}             
\usepackage{type1cm}           
\usepackage{makeidx}            
\usepackage{graphicx}        
\usepackage{multicol}         
\usepackage[bottom]{footmisc} 
   
\usepackage{epsfig}   
\usepackage{psfrag,rotating}    
\usepackage{amssymb,floatflt,enumerate}  
\usepackage{amsmath,amscd,psfrag,leqno} 
\usepackage[mathscr]{eucal}  
\usepackage{shadethm}           
\usepackage{graphicx,subfigure}   
\usepackage{overpic,contour}      
\contourlength{0.3mm}  
 
\def\Beweisende{\square}             
\def\BewEnde{\hfill{\Beweisende}}

\def\phm{{\hphantom{-}}}

\def\phi{\varphi}

\def\RR{{\mathbb R}}

\def\CC{{\mathbb C}}

\newcommand{\kreuz}{\!\times}

\def\Vkt#1{{\mathbf #1}}

\newcommand{\go}[1]{{\sf #1}}

\makeindex
\begin{document}  

\title*{On the line-symmetry of self-motions of linear pentapods}
\author{Georg Nawratil}
\institute{
  Georg Nawratil \at
  Institute of Discrete Mathematics and Geometry, Vienna University of Technology, Austria, \\
  \email{nawratil@geometrie.tuwien.ac.at}}
\authorrunning{G.\ Nawratil}

\maketitle

\abstract{
We show that all self-motions of pentapods with linear platform of Type 1 and Type 2 can be generated by 
line-symmetric motions. Thus this paper closes a gap between the more than 100 year old works  
of Duporcq and Borel and the extensive study of line-symmetric motions done by Krames in the 1930's. 
As a consequence we also get a new solution set for the Borel Bricard problem. 
Moreover we discuss the reality of self-motions  and give a sufficient condition for the 
design of linear pentapods of Type 1 and Type 2, which have a self-motion free workspace. 
} 


\section{Introduction}
 
The geometry of a linear pentapod is given by the five base anchor points $\go M_i$ in the fixed system $\Sigma_0$ 
and by the five collinear platform anchor points $\go m_i$ in the moving system $\Sigma$ (for $i=1,\ldots ,5$). 
Each pair $(\go M_i,\go m_i)$ of corresponding anchor points is connected by a SPS-leg, where only the prismatic joint is active.

If the geometry of the manipulator is given as well as the lengths $R_i$ of the five pairwise distinct legs, a linear 
pentapod has generically mobility 1, which corresponds to the rotation about 
the carrier line $\go p$ of the five platform anchor points.
As this rotational motion is irrelevant for applications with axial symmetry (e.g.\ 
5-axis milling, spot-welding, laser or water-jet engraving/cutting, 
spray-based painting, etc.), these mechanisms are of great practical interest. 
Nevertheless configurations should be avoided where the manipulator gains an additional 
uncontrollable mobility,  which is referred as self-motion. 

\subsection{Review on self-motions of linear pentapods}\label{review}

The self-motions of linear pentapods represent interesting solutions to the 
still unsolved problem posed by the French Academy of Science for the {\it Prix Vaillant} of the year 1904, which is also known as 
Borel-Bricard  problem (cf.\ \cite{borel,bricard,husty_bb}) and reads as follows: 
{\it "Determine and study all displacements of a rigid body in which distinct points of the body move on spherical paths."}

For the special case of five collinear points the Borel-Bricard problem was studied by  
Darboux \cite[page 222]{koenigs}, Mannheim \cite[pages 180ff]{mannheim} and Duporcq \cite{duporcq} 
(see also Bricard \cite[Chapter III]{bricard}). 
A contemporary and accurate reexamination of these old results, 
which also takes the coincidence of platform anchor points into account, was done 
in \cite{linear_penta} yielding a  full classification of linear pentapods with self-motions.

Beside the architecturally singular linear pentapods \cite[Corollary 1]{linear_penta} and some trivial cases with pure 
rotational self-motions \cite[Designs $\alpha$, $\beta$, $\gamma$]{linear_penta}
or pure translational ones \cite[Theorem 1]{linear_penta} there only remain the following three designs:

Under a self-motion each point of the line $\go p$ has a spherical (or planar) trajectory. 
The locus of the corresponding sphere centers is a space curve $\go P$ of degree 3, where the 
mapping from $\go p$ to $\go P$ is named $\sigma$.
$\go P$  intersects the ideal plane in one 
real point $\go W$ and two conjugate complex ideal points, where the latter ones are the cyclic points $\go I$ and $\go J$ of a plane orthogonal to the 
direction of $\go W$. $\go P$ is therefore a so-called straight cubic circle. 
The following subcases can be distinguished:
\begin{enumerate}[$\bullet$]
\item
$\go P$ is irreducible: 
	\begin{enumerate}[$-$]
	\item
	$\sigma$ maps the ideal point $\go U$ of $\go p$ to $\go W$ (Type 5 according to \cite{linear_penta}).
	\item
	$\sigma$ maps $\go U$ to a finite point of $\go P$ (Type 1 according to \cite{linear_penta}; see Fig.\ \ref{fig-0}a).
	\end{enumerate}
\item
$\go P$ splits up into a circle $\go q$ and a line $\go s$, which is orthogonal to the carrier plane $\varepsilon$ of $\go q$ and 
intersects $\go q$ in a point $\go Q$. There is a bijection $\sigma$ between $\go p\setminus\left\{ \go S \right\}$ and 
$\go q\setminus\left\{ \go Q \right\}$; moreover the finite point $\go S$ is mapped to $\go s$. 
As a consequence $\sigma$ maps $\go U$ to a point on the circle different from $\go Q$ 
(Type 2 according to \cite{linear_penta}; see Fig.\ \ref{fig-0}b). 

\end{enumerate}

\begin{figure}[top]
\begin{center} 
\subfigure[]{ 
 \begin{overpic}
    [width=47mm]{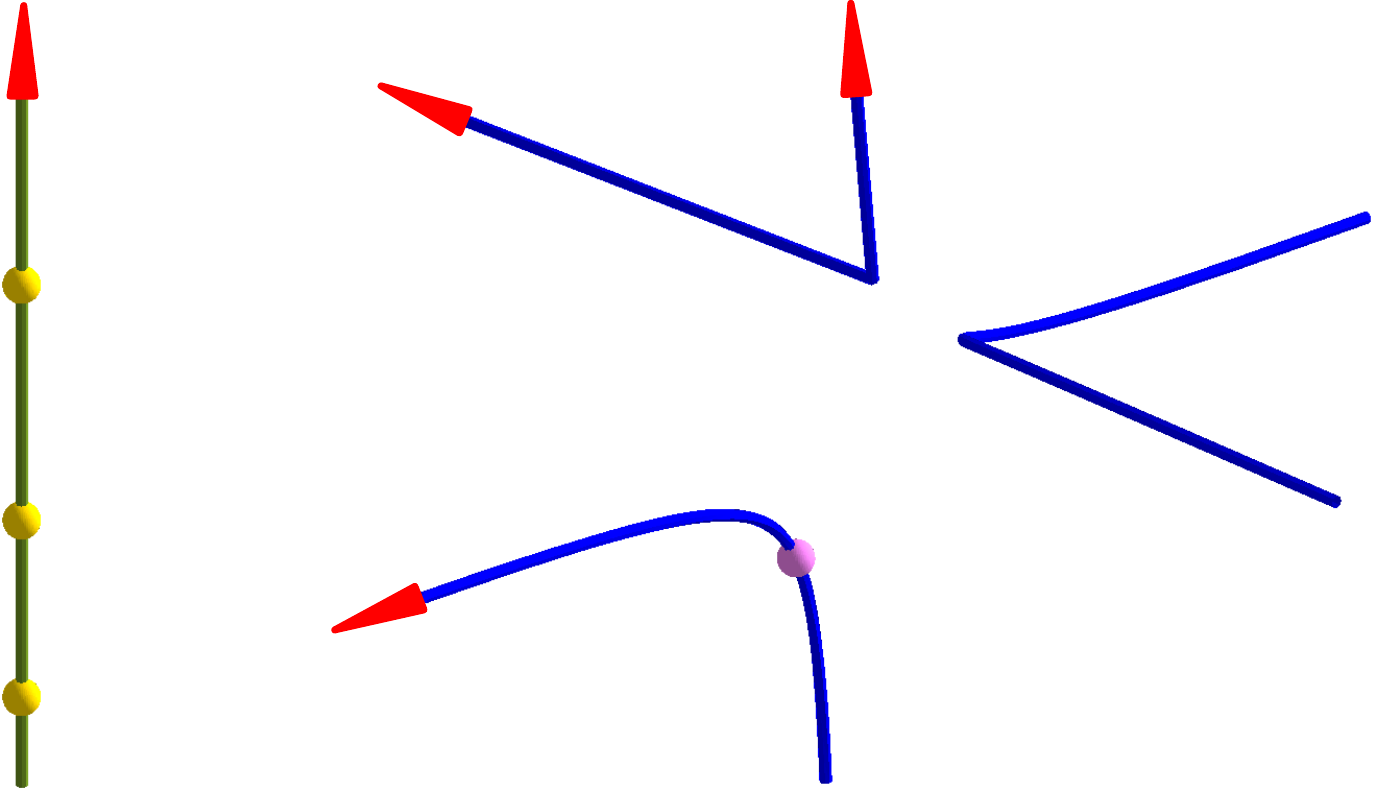}
\begin{small}
\put(3,0){$\sigma^{-1}(\go W)$}
\put(3,19){$\sigma^{-1}(\go J)$}
\put(3,30){$\sigma^{-1}(\go I)$}
\put(3,43){$\go p$}
\put(4,52){$\go U$}
\put(30,52){$\go I$}
\put(64.5,52){$\go J$}
\put(28,6){$\go W$}
\put(60.5,14){$\sigma(\go U)$}
\put(94,14.5){$\go P$}
\end{small}     
  \end{overpic} 
 }
\hfill
\subfigure[]{
\begin{overpic}
    [width=50mm]{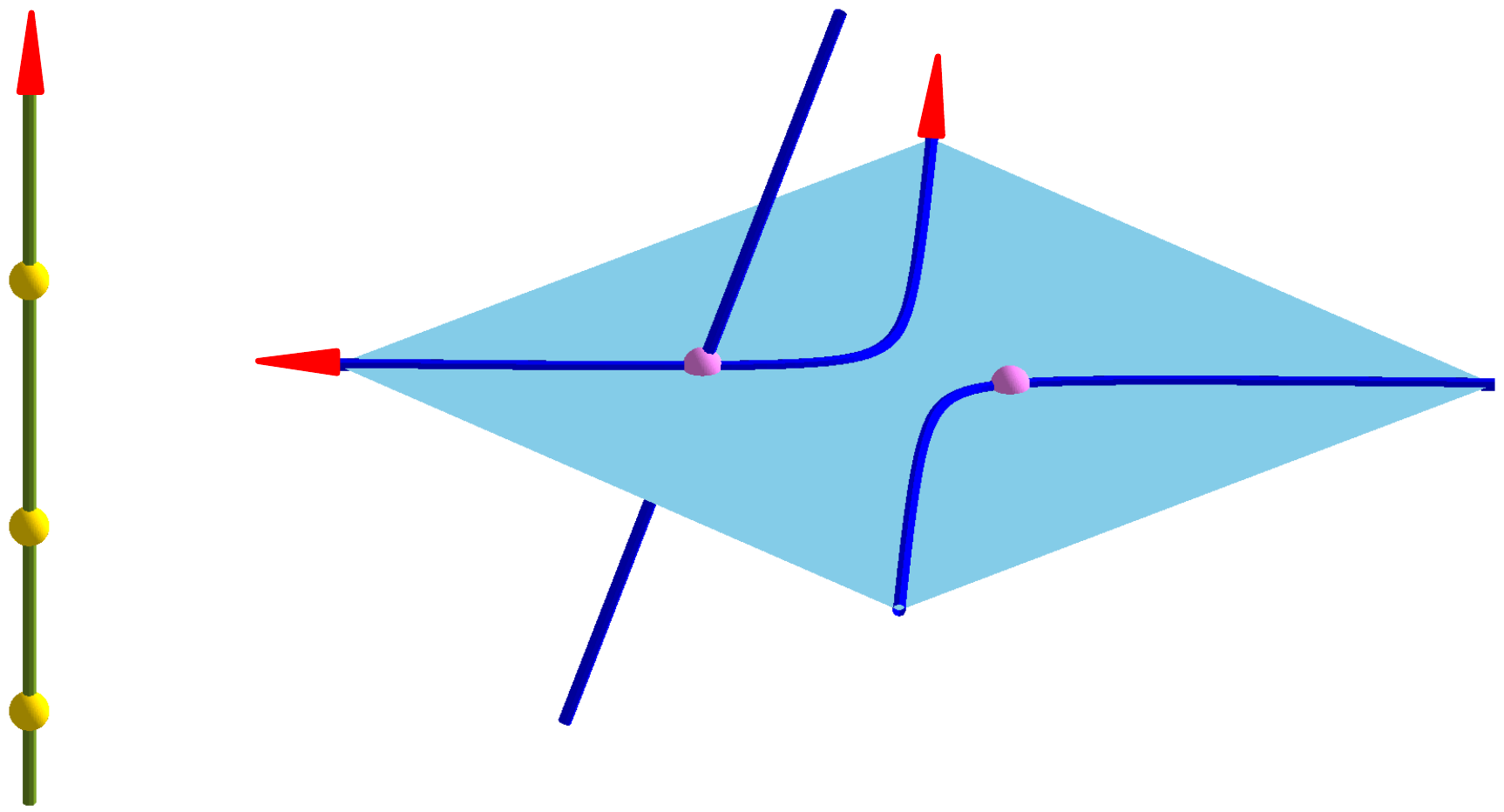}
\begin{small}
\put(3.5,0){$\sigma^{-1}(\go J)$}
\put(3.5,12){$\go S$}
\put(3.5,36){$\sigma^{-1}(\go I)$}
\put(3.5,26){$\go p$}
\put(4,48.5){$\go U$}
\put(32.5,6){$\go s$}
\put(65,22){$\sigma(\go U)$}
\put(17,23){$\go I$}
\put(64,46){$\go J$}
\put(44.5,23.3){$\go Q$}
\put(55.5,16.5){$\go q$}
\put(72,35){$\varepsilon$}
\end{small}         
  \end{overpic}
} 
\end{center} 
\caption{Projective sketch of linear pentapods of (a) Type 1 and (b) Type 2, respectively, with self-motions. 
}
  \label{fig-0}
\end{figure}
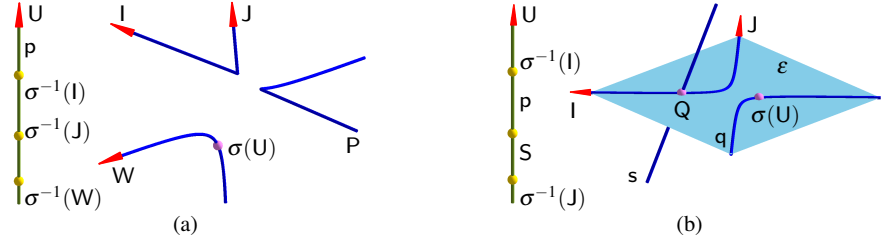

\subsection{Basics on line-symmetric motions}\label{blsm}

Krames (e.g.\ \cite{krames,krames2}) studied special one-parametric motions ({\it Symmetrische Schrotung} in German), which are obtained
by reflecting the moving system $\Sigma$ in the generators of a ruled surface 
of the fixed system $\Sigma_0$, which is the so called {\it basic surface}.  
These so-called {\it line-symmetric motions} were also studied by Bottema and Roth \cite[\S 7 of Chapter 9]{bottema}, 
who gave an intuitive algebraic characterization in terms of 
Study parameters $(e_0:e_1:e_2:e_3:f_0:f_1:f_2:f_3)$, which are shortly repeated next. 

All real points of the Study parameter space $P^7$ (7-dimensional projective space), 
which are located on the so-called Study quadric $\Psi:\,\sum_{i=0}^3e_if_i=0$, 
correspond to an Euclidean displacement with exception of the 3-dimensional subspace $e_0=e_1=e_2=e_3=0$, 
as its points cannot fulfill the condition $N\neq 0$ with $N:=e_0^2+e_1^2+e_2^2+e_3^2$. 
The translation vector $\Vkt s:=(s_1,s_2,s_3)^T$ and the rotation matrix $\Vkt R$ of the corresponding 
Euclidean displacement $\Vkt m_i\mapsto\Vkt R\Vkt m_i + \Vkt s$ are given for $N=1$ by:  
\begin{equation*}
\begin{split}
s_1&=-2(e_0f_1-e_1f_0+e_2f_3-e_3f_2), \quad
s_2=-2(e_0f_2-e_2f_0+e_3f_1-e_1f_3), \\
s_3&=-2(e_0f_3-e_3f_0+e_1f_2-e_2f_1), \\
\Vkt R &= 
\begin{pmatrix}  
r_{11} & r_{12} & r_{13} \\
r_{21} & r_{22} & r_{23} \\
r_{31} & r_{32} & r_{33} 
\end{pmatrix}=
\begin{pmatrix} 
e_0^2+e_1^2-e_2^2-e_3^2 & 2(e_1e_2-e_0e_3) & 2(e_1e_3+e_0e_2)  \\
2(e_1e_2+e_0e_3) & e_0^2-e_1^2+e_2^2-e_3^2 & 2(e_2e_3-e_0e_1)  \\
2(e_1e_3-e_0e_2) & 2(e_2e_3+e_0e_1) & e_0^2-e_1^2-e_2^2+e_3^2
\end{pmatrix}.
\end{split}
\end{equation*}  
  
There always exists a moving frame (in dependence of a given fixed frame) in a way  
that $e_0=f_0=0$ holds for a line-symmetric motion. 
Then $(e_1:e_2:e_3:f_1:f_2:f_3)$ are the Pl\"ucker coordinates (according to the convention used in \cite{bottema}) 
of the generators of the basic surface with respect to the fixed frame .

\subsection{Line-symmetric self-motions of linear pentapods}\label{zurueck}

It is well known (cf.\ \cite[\S 15]{duporcq}, \cite[\S 12]{bricard}) that the self-motions of Type 5 are 
obtained by restricting the Borel-Bricard motions\footnote{These are the only non-trivial motions
where all points of the moving space have spherical trajectories (cf.\ \cite[Chapter VI]{bricard}).} (also known as BB-I motions) 
to a line. Note that this special case was also discussed in detail by Krames \cite[Section 5]{krames}, who 
also pointed out the line-symmetry of BB-I motions. 

Beside these BB-I motions, 
there also exist line-symmetric motions (so-called BB-II motions),  
where all points of a hyperboloid carrying two reguli of lines have spherical trajectories. 
It is known (cf.\ \cite[page 24]{hartmann} and \cite[page 188]{krames2}) that the corresponding sphere centers of lines, belonging to one 
regulus\footnote{The corresponding sphere centers of lines belonging to the other regulus are again 
located on a line (cf.\ \cite[page 24]{hartmann}), which imply  architecturally singular designs of linear pentapods.},
form irreducible straight cubic circles, which imply examples of self-motions of Type 1. 
Note that there also exist degenerated cases where the hyperboloid splits up into two orthogonal planes, 
which contain examples of self-motions of Type 2.

A simple count of free parameters shows that not all self-motions of Type 1 (5-parametric set\footnote{With respect to the notation introduced 
in Section 2 these five parameters are 
$C,a_r,a_c,a_4$ and $p_5$ or $R_1$ (cf.\ Eq.\ (\ref{zwei}))  by canceling the factor of similarity by setting $A=1$.} of motions where all
points of a line have spherical paths) can be generated by BBM-II motions (which produce only a 4-parametric set\footnote{These are the parameters 
$a,c,g,k$ used in \cite[Section 2.3]{hartmann}.}). 
The same argumentation holds for Type 2 self-motions and the mentioned degenerated case.

As a consequence the question arise whether all self-motions of linear pentapods of Type 1 and Type 2 can be generated by 
line-symmetric motions. If this is the case we can apply a construction proposed by Krames \cite[page 416]{krames}, which is 
discussed in Section \ref{new:solutions}, yielding new solutions to the Borel-Bricard problem.

\section{One the line-symmetry of Type 1 and Type 2 self-motions}

For our calculations we do not select arbitrary pairs $(\go m_i,\go M_i)$ of $\go p$ and $\go P$, which are in correspondence with 
respect to $\sigma$ ($\Leftrightarrow$ $\sigma(\go m_i)=\go M_i$), but choose the following special ones:

$\go M_4$ equals $\go W$, $\go M_2$ coincides with $\go I$ and $\go M_3$ with $\go J$. The corresponding platform anchor points are denoted by 
$\go m_4$, $\go m_2$ and $\go m_3$, respectively. As $\go M_i$ are ideal points the corresponding points $\go m_i$ are not running on spheres 
but in planes orthogonal to the direction of $\go M_i$. 
Therefore these three point pairs imply three so-called Darboux conditions $\Omega_i$ for $i=2,3,4$. 
Moreover we denote $\go U$ as $\go m_5$ and its corresponding finite point under $\sigma$ by $\go M_5$. This point pair describes 
a so-called Mannheim condition $\Pi_5$ (which is the inverse of a Darboux condition). 
The pentapod is completed by a sphere condition $\Lambda_1$ of any pair of corresponding finite points $\go m_1$ and 
$\go M_1$. 

In \cite{linear_penta} we have chosen the fixed frame $\mathcal{F}_0$ in a way that 
$\go M_1$ equals its origin and $\go M_4$ coincides with the ideal point of the $z$-axis. 
Moreover we located the moving frame $\mathcal{F}$ in a way that $\go p$ coincides with the $x$-axis, where $\go m_1$ 
equals its origin.

For the study at hand it is advantageous to select a different set of fixed and moving frames 
$\mathcal{F}_0^{\prime}$ and $\mathcal{F}^{\prime}$, respectively:
\begin{enumerate}[$\bullet$]
\item
As $\go M_2$ and $\go M_3$ coincides with the cyclic points, 
we can assume without loss of generality (w.l.o.g.) that $\go M_5$ is located in the $xz$-plane (as a rotation about the z-axis does not 
change the coordinates of $\go M_1,\ldots ,\go M_4$). 
Moreover we want to apply a translation in a way that $\go M_5$ is in the origin of the new fixed frame $\mathcal{F}_0^{\prime}$. 
Summed up the coordinates  with respect to $\mathcal{F}_0^{\prime}$ read as:
\begin{equation}
 \Vkt M_5=(0,0,0),\quad \Vkt M_1=(A,0,C) \quad \text{with} \quad A\neq 0 
\end{equation}
as $A=0$ implies a contradiction to the properties of $\go P$ for Type 1 and Type 2 pentapods given in Section \ref{review}. 
Moreover, $\go M_2$, $\go M_3$ and $\go M_4$ are the ideal points in direction $(1,i,0)^T$, $(1,-i,0)^T$ and $(0,0,1)^T$, respectively. 
\item
With respect to $\mathcal{F}_0^{\prime}$ the location of $\go p$ is undefined, but the 
coordinates $\Vkt m_i$ of $\go m_i$ can be parametrized as follows for $i=1,\ldots, 4$: 
\begin{equation}
\Vkt m_i= \Vkt n+ (a_i-a_r)\Vkt d  \quad\text{with}\quad a_1=0,\,\, a_2=a_r+ia_c,\,\,  a_3=a_r-ia_c
\end{equation}
where $a_r,a_c\in\RR$ and $a_c\neq 0$ holds. 
$\go m_5$ is the ideal point in direction of the unit-vector $\Vkt d=(d_1,d_2,d_3)^T$, which obtains the 
rational homogeneous parametrization of the unit-sphere, i.e.\:
\begin{equation}
d_1=\tfrac{2h_0h_1}{h_0^2+h_1^2+h_2^2}, \quad
d_2=\tfrac{2h_0h_2}{h_0^2+h_1^2+h_2^2},  \quad 
d_3=\tfrac{h_1^2+h_2^2-h_0^2}{h_0^2+h_1^2+h_2^2}.
\end{equation}
\end{enumerate}
\noindent
Now we are looking for the point $\Vkt n=(n_1,n_2,n_3)^T$ and the direction $(h_0:h_1:h_2)$ in a way that 
for the self-motion of the pentapod $e_0=f_0=0$ holds. 
We can discuss Type 1 and Type 2 at the same time, just having in mind that $a_4\neq 0\neq C$  has to hold for Type 1 
and $a_4=0=C$ for Type 2 (according to \cite{linear_penta}).

By setting $\Vkt r_i:=(r_{i1},r_{i2},r_{i3})^T$ for $i=1,2,3$ the Darboux and Mannheim constraints with respect to $\mathcal{F}_0^{\prime}$ and $\mathcal{F}^{\prime}$
can be written as: 
\begin{align}\label{oben}
\Omega_2:\,&(s_1+\Vkt r_1\Vkt m_2)-i(s_2+\Vkt r_2\Vkt m_2)-p_2N=0, &\quad \Omega_4:\,&(s_3+\Vkt r_3\Vkt m_4)-p_4N=0, \\ \label{unten}
\Omega_3:\,&(s_1+\Vkt r_1\Vkt m_3)+i(s_2+\Vkt r_2\Vkt m_3)-p_3N=0, &\quad \Pi_5:\,&(\Vkt R\Vkt d)(\Vkt s+\Vkt R\Vkt p_5)N^{-1}=0, 
\end{align}
with $\Vkt p_5=\Vkt n+(p_5-a_r)\Vkt d$, which is the coordinate vector of the intersection point of the Mannheim plane and $\go p$ with respect to $\mathcal{F}^{\prime}$. 
Moreover $(p_j,0,0)^T$ for $j=2,3$ (resp.\ $(0,0,p_4)^T$) are the coordinates of the intersection point of the Darboux plane and the $x$-axis (resp.\ $z$-axis) of $\mathcal{F}_0^{\prime}$. 

\begin{remark}
As from the Mannheim constraint $\Pi_5$ of Eq.\ (\ref{unten}) 
the factor $N$ cancels out, all four constraints $\Omega_2,\Omega_3,\Omega_4,\Pi_5$ are homogeneous quadratic in the Study parameters and 
especially linear in $f_0,\ldots ,f_3$. \hfill $\diamond$
\end{remark}

According to \cite[Theorems 13 and 14]{linear_penta} the leg-parameters $p_2,\ldots ,p_5,R_1$ 
have to fulfill the following necessary and sufficient conditions for the self-mobility (over $\CC$) of 
a linear pentapod of Type 1 and Type 2, respectively:
\begin{align}\label{eins}
&p_2=\tfrac{Aa_3v}{(a_3-a_4)^2},\qquad 
p_3=\tfrac{Aa_2v}{(a_2-a_4)^2},\qquad 
p_4=-\tfrac{Ca_4v}{(a_2-a_4)(a_3-a_4)}, \\ \label{zwei} 
&\phm(a_2-a_4)^2(a_3-a_4)^2\left[2wp_5-vR_1^2-(2w-va_4)a_4 \right] + vw^2(A^2+C^2)=0,
\end{align}
with $v:=a_2+a_3-2a_4$ and $w:=a_2a_3-a_4^2$. 
Therefore if we set $p_2,p_3,p_4$ as given in Eq.\ (\ref{eins}) then only one condition 
in $p_5$ and $R_1$ remains in  Eq.\ (\ref{zwei}). Therefore these pentapods have a 1-dimensional
set of self-motions.

\begin{theorem}\label{main}
Each self-motion of a linear pentapod of Type 1 and Type 2 can be generated by a 1-dimensional set of 
line-symmetric motions. For the special case $p_5=a_4=a_r$ this set is even 2-dimensional.
\end{theorem}

\begin{proof}
W.l.o.g.\ we can set $e_0=0$ as any two directions $\Vkt d$ of $\go p$ can be transformed into each other 
by a half-turn about their enclosed bisecting line. Note that this line is not uniquely determined if and only if 
the two directions are antipodal.

W.l.o.g.\ we can solve $\Psi,\Omega_2,\Omega_3,\Omega_4$ for $f_0,f_1,f_2,f_3$ and plug the obtained 
expressions into $\Pi_5$, which yields in the numerator a homogeneous quartic polynomial $G[1563]$ in $e_1,e_2,e_3$, 
where the number in the brackets gives the number of terms. Moreover the numerator of the obtained expression for $f_0$
is denoted by $F[600]$, which is a homogeneous cubic polynomial in \medskip $e_1,e_2,e_3$. 

\noindent
{\bf General Case $(v\neq 0)$:} The condition $G=0$ already expresses the self-motion as $G$ equals $\Lambda_1$ if 
we solve Eq.\ (\ref{zwei}) for $R_1$. Moreover $F=0$ has to hold if the self-motion of the line $\go p$ can be generated by a line-symmetric motion. 
As for any solution $(e_1:e_2:e_3)$ of $F=0$ also $G=0$ has to hold, $G$ has to split into $F$ and a homogeneous 
linear factor $L$ in $e_1,e_2,e_3$. 

Now $L=0$ cannot correspond to a self-motion of the linear pentapod, 
but has to arise from the ambiguity in representing a direction of $\go p$ 
mentioned at the beginning of the proof. This can be argued indirectly as follows: 

Assumed $L=0$ implies a self-motion, then it has to be a Sch\"onflies motion (with a certain direction $\go v$ of the rotation axis) 
due to $e_0=0$.  
As under such a motion the angle enclosed by  $\go v$ and $\go p$ remains constant\footnote{This angle condition 
can be seen as the limit of the sphere condition (cf.\ \cite[Section 4.1]{duporcq_n}).}
the ideal point $\go U$ of $\go p$ has to be mapped by $\sigma$ to the ideal point $\go V$ of $\go v$. 
This implies that $\go V$ has to coincide with $\go W$, which can only be the case for pentapods of Type 5; a 
contradiction. 

Therefore there has to exist a pose of $\go p$ during the self-motion, 
where it is oppositely oriented with respect to the fixed frame and moving frame, respectively. 
As a consequence we can set $L=d_1e_1+d_2e_2+d_3e_3$ which yields the ansatz:
\begin{equation}
\Delta: \quad \lambda LF-G=0.
\end{equation}

The resulting set of four equations arising from the coefficients of 
$e_1^3e_2$, $e_1^3e_3$, $e_1e_3^3$ and $e_2e_3^3$ of $\Delta$ has the unique solution: 
\begin{equation}\label{ns}
n_1=a_cd_2,\quad
n_2=-a_cd_1, \quad
n_3=(a_r-a_4)d_3, \quad
\lambda=2(h_0^2+h_1^2+h_2^2).
\end{equation}
Now $\Delta$ splits up into $(e_1^2+e_2^2+e_3^2)^2(h_0^2+h_1^2+h_2^2)H[177]$, where $H$ is homogeneous of degree 4 in $h_0,h_1,h_2$. 
For the explicit expression of the planar quartic curve $H=0$ see Remark \ref{remH}, which is given right after this proof. 

\begin{remark}\label{rem:par}
Note that all self-motions of the general case can be parametrized as the resultant of 
$G$ and the normalizing condition $N-1$ with respect to 
$e_i$ yields a polynomial, which is only quadratic in $e_j$ for  pairwise distinct $i, j\in\left\{1,2\right\}$.
\hfill $\diamond$
\end{remark}

\noindent
{\bf Special Case $(v= 0)$:}
If $v=0$ holds, we cannot solve Eq.\ (\ref{zwei}) for  $R_1$. 
The conditions $v=0$ and  Eq.\ (\ref{zwei}) imply  $p_5=a_4=a_r$. 
Now $G$ is fulfilled identically and the self-motion is given by $\Lambda_1=0$, which is of degree 4 in $e_1,e_2,e_3$.
Moreover for this special case $F=0$ already holds for $\Vkt n$ given 
in Eq.\ (\ref{ns}). Therefore any direction $(h_0:h_1:h_2)$ for $\go p$ can be chosen 
in order to fix the line-symmetric motion. \hfill $\BewEnde$
\end{proof}

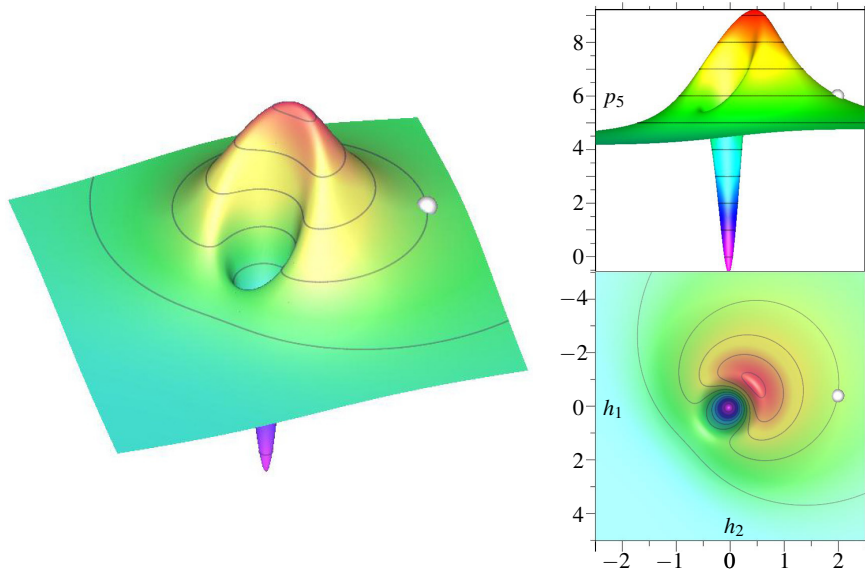
\begin{figure}[top] 
 \begin{minipage}{73mm} 
\begin{center}
 \contourlength{0.5mm} 
 \begin{overpic}
    [width=70mm]{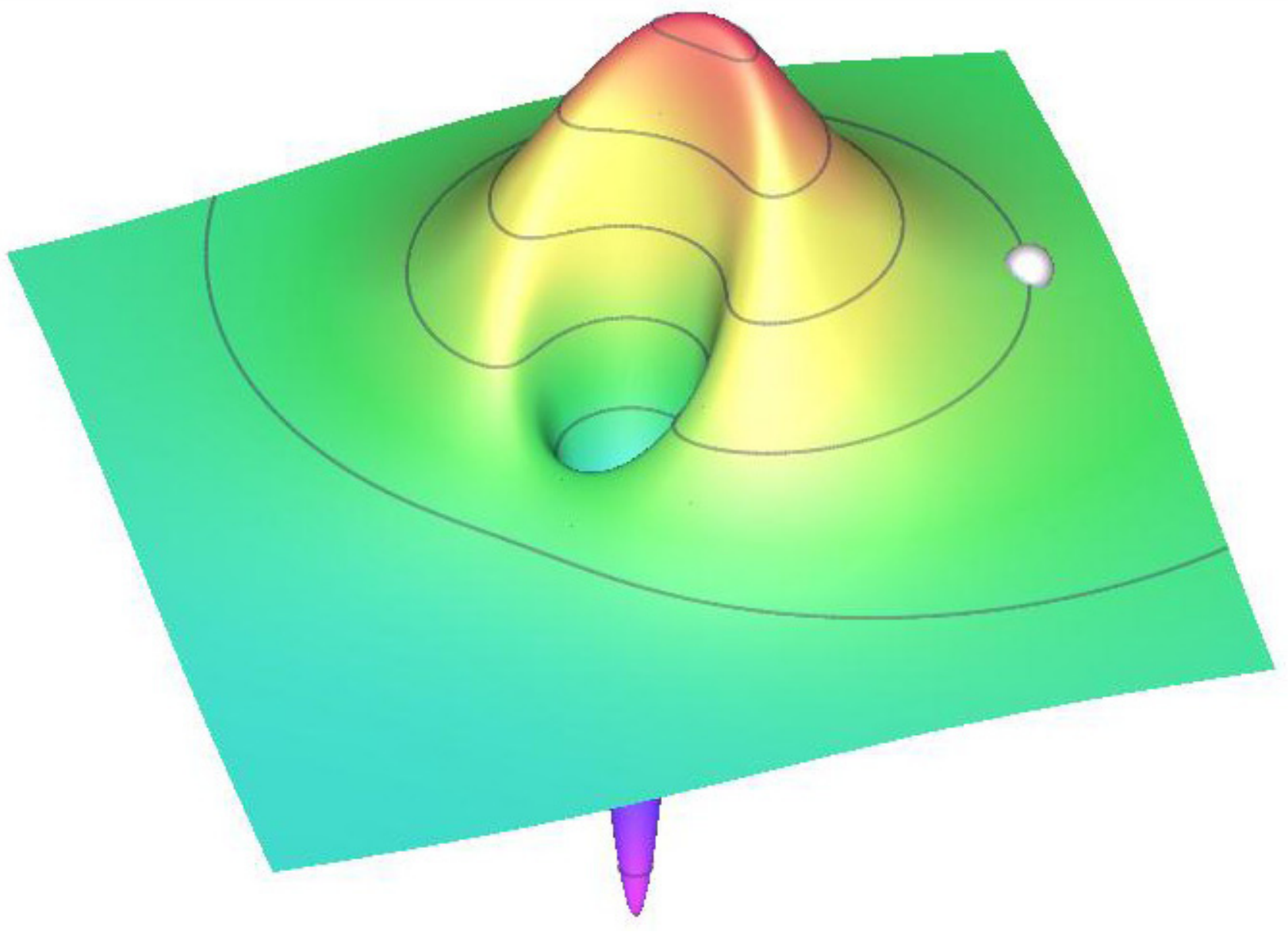}
  \end{overpic} 
\end{center}
 \end{minipage}    
\hfill
 \begin{minipage}{40mm}
\begin{center}
 \contourlength{0.3mm}
 \begin{overpic}
    [width=37mm]{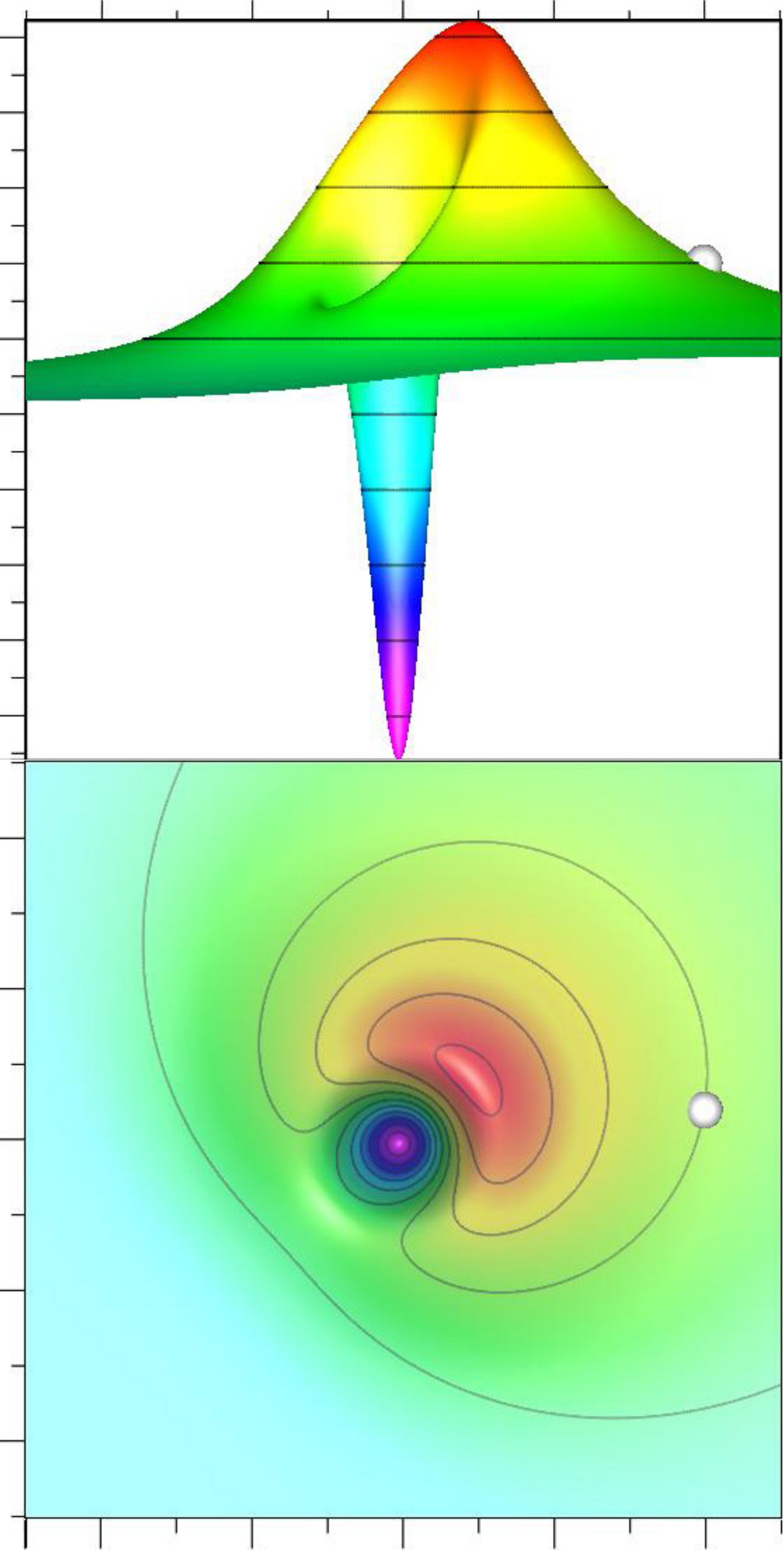}
\begin{small}
\put(25,-3){$0$}
\put(35,-3){$1$}
\put(44.75,-3){$2$}
\put(12.5,-3){$-1$}
\put(2.75,-3){$-2$}
\put(25,-3){$0$}
\put(25,3.5){$h_2$}
\put(3,25){$h_1$}
\put(-2.5,25){$0$}
\put(-2.5,15.25){$2$}
\put(-2.5,5.5){$4$}
\put(-5,35){$-2$}
\put(-5,45){$-4$}
\put(-2.5,52.35){$0$}
\put(-2.5,62.3){$2$}
\put(-2.5,72.25){$4$}
\put(-2.5,81.7){$6$}
\put(-2.5,91.3){$8$}
\put(3,81.7){$p_5$}
\end{small}         
  \end{overpic}
\end{center}
\vspace{+0mm} 
\end{minipage}
\caption{
For the example given in Eq.\ (\ref{par:ex}) the graph of $p_5$ in dependency of $h_1$ and $h_2$ with $h_0=1$ is displayed in 
the axonometric view on the left and in the front resp.\ top view on the right side. 
The highlighted point at height $6$ corresponds to the values $h_1=-\tfrac{489262}{226525}+\tfrac{488}{226525}\sqrt{675091}$ and 
$h_2=\tfrac{535336}{226525}+\tfrac{446}{226525}\sqrt{675091}$. 
}
\label{fig-3}
\end{figure}

\begin{remark}\label{remH}
$H=0$ represents a planar quartic curve, which can be verified to be entirely circular. 
Moreover $H=0$ can be solved  linearly for $p_5$. The corresponding graph is illustrated in Fig.\ \ref{fig-3}. 
  
If we reparametrize the $h_0h_1h_2$-plane in terms of homogenized polar coordinates by:
\begin{equation}
h_0=(\tau_1^2+\tau_0^2)\rho_0,\qquad
h_1=(\tau_1^2-\tau_0^2)\rho_1,\qquad
h_2=2\tau_0\tau_1\rho_1,
\end{equation}
where $(\tau_0,\tau_1)\neq (0,0)\neq (\rho_0,\rho_1)$ and $\tau_0,\tau_1,\rho_0,\rho_1\in \RR$ hold, then  
$H$ factors into $(\tau_0^2+\tau_1^2)^3(H_2\tau_1^2 + H_1\tau_0\tau_1 +H_0\tau_0^2)$ with 
\begin{equation}
\begin{split}
H_1&=8\rho_0\rho_1A(a_4-a_r)(\rho_1^2+\rho_0^2)(a_r^2-a_4^2+a_c^2)a_c, \\
H_0-H_2&= 8\rho_0\rho_1A(a_4-a_r)(\rho_1^2+\rho_0^2)[a_r(a_r-a_4)^2+a_c^2(a_r-2a_4)], \\
H_0+H_2&= 2\left[(a_r-a_4)^2+a_c^2\right]
[2a_4(\rho_1^4-\rho_0^4)(a_4-a_r)C +  \\
&\left((a_r-a_4)^2+a_c^2\right)\left((\rho_0^4+\rho_1^4)(a_4-p_5)+2\rho_0^2\rho_1^2(2a_r-a_4-p_5)  \right)].
\end{split}
\end{equation}
Therefore this equation can be solved quadratically for the homogeneous parameter $\tau_0:\tau_1$. 
Note that the value $p_5$ is fixed during a self-motion.
\hfill $\diamond$
\end{remark}

\section{On the reality of Type 1 and Type 2 self-motions}

A similar computation to \cite[Example 1]{linear_penta} shows that for any real point $\go p_t\in \go p$  
with $t\in \RR$ and coordinate vector $\Vkt p_t=\Vkt n+ (t-a_r)\Vkt d$  with respect to $\mathcal{F}^{\prime}$  
the corresponding real point $\go P_t\in\go P$ has the following coordinate vector $\Vkt P_t$ with respect to $\mathcal{F}_0^{\prime}$:
\begin{equation}\label{Pt}
\Vkt P_t=\left(
\tfrac{A(a_r^2+a_c^2-ta_r)}{(t-a_r)^2+a_c^2},
-\tfrac{Aa_ct}{(t-a_r)^2+a_c^2},
\tfrac{Ca_4}{a_4-t}\right)^T.
\end{equation}
As $L=0$ corresponds with one configuration of the self-motion we can compute the locus $\mathcal{E}_t$
of $\go p_t$  with respect to $\mathcal{F}_0^{\prime}$ under the 1-parametric 
set of self-motions by the variation of $(h_0:h_1:h_2)$ within $L=0$. 
Moreover due to the mentioned ambiguity we can select an arbitrary solution $(e_0:e_1:e_2)$ for $L=0$ 
fulfilling the normalization condition 
$N=1$; e.g.:
$e_1={h_2}{(h_1^2+h_2^2)^{-\tfrac{1}{2}}}$, 
$e_2=-{h_1}{(h_1^2+h_2^2)^{-\tfrac{1}{2}}}$ and 
$e_3=0$. 
Now the computation of $\Vkt R\Vkt p_t+\Vkt s$ yields a rational quadratic parametrization of $\mathcal{E}_t$ 
(see Fig.\ \ref{fig-1}a) in dependency of $(h_0:h_1:h_2)$. 

Note that this approach also includes the special case $(v= 0)$ as there always exists a value for $R_1^2$ 
(in dependency of $(h_0:h_1:h_2)$) in a way that $\Lambda_1=0$ holds.

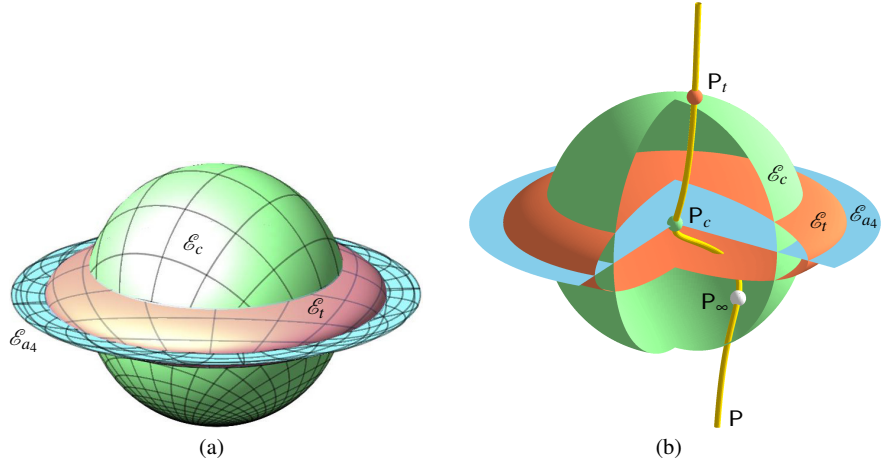
\begin{figure}[top]
\begin{center} 
\subfigure[]{ 
 \begin{overpic}
    [width=55mm]{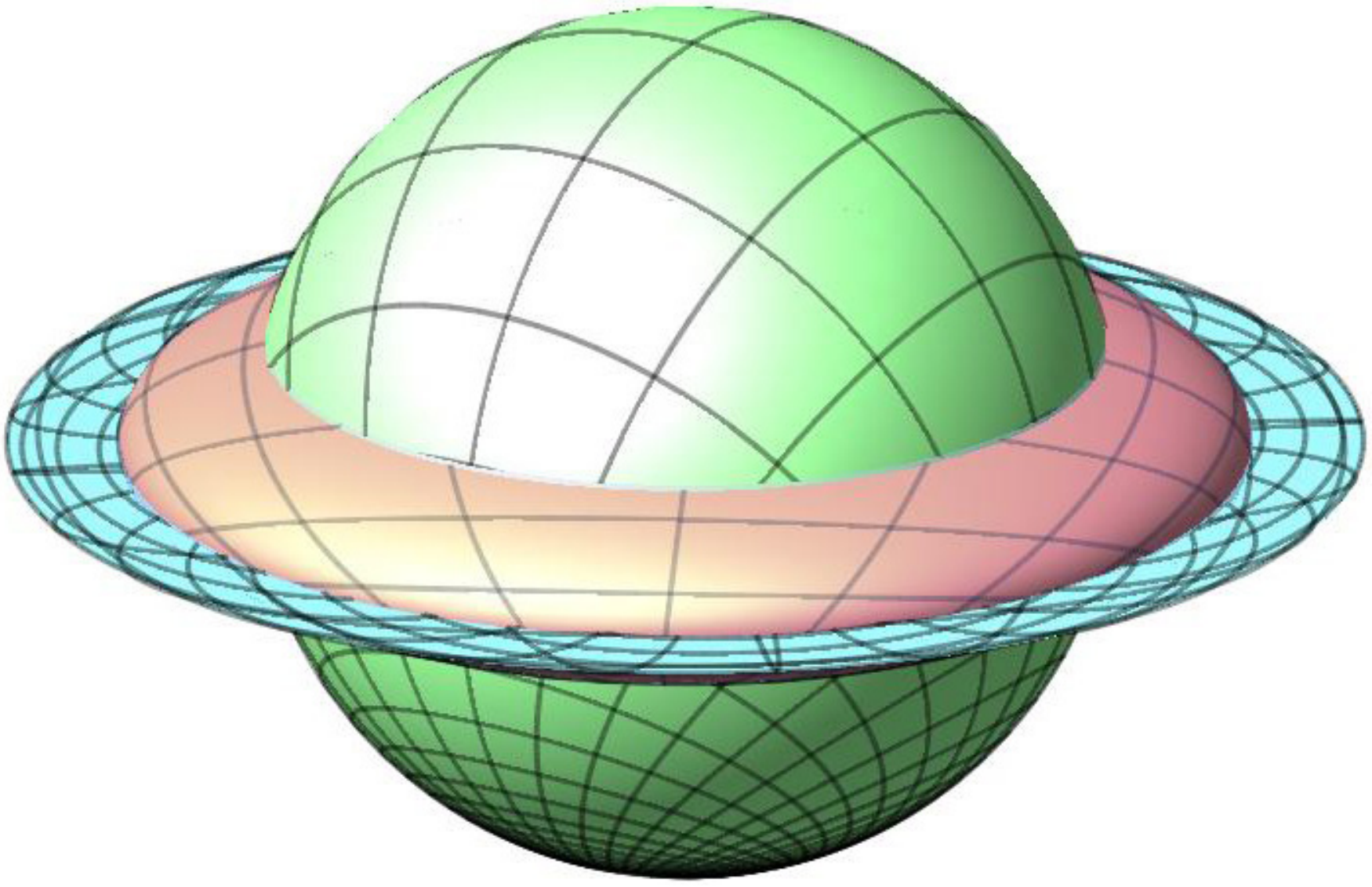}
\begin{small}
\put(72,27.5){$\mathcal{E}_{t}$}
\put(42,43){$\mathcal{E}_{c}$}
\put(0,20){$\mathcal{E}_{a_4}$}
\end{small}     
  \end{overpic} 
 }
\hfill
\subfigure[]{
\begin{overpic}
    [width=55mm]{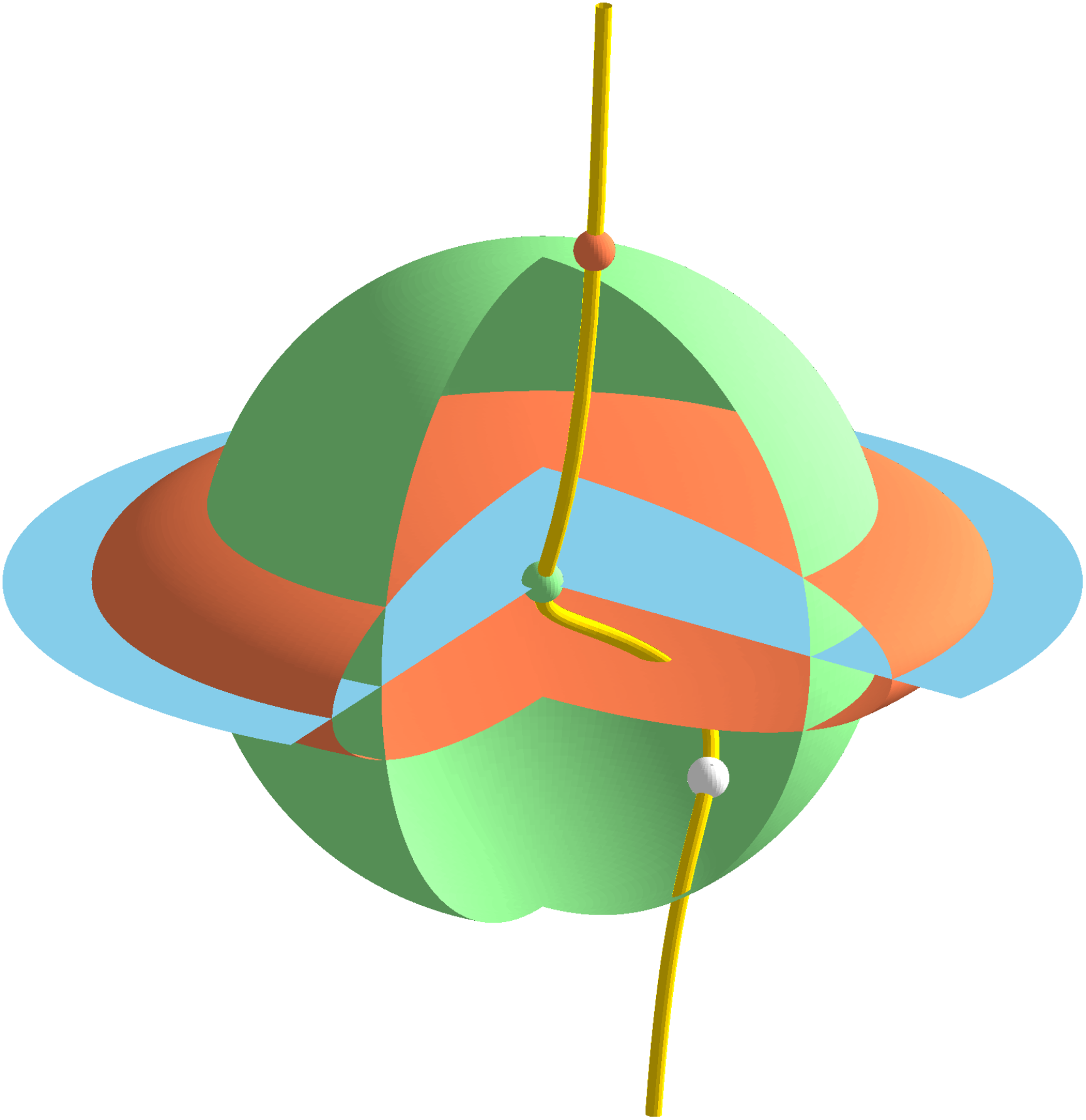}
\begin{small}
\put(80,47){$\mathcal{E}_{t}$}
\put(70,58){$\mathcal{E}_{c}$}
\put(89.3,48.5){$\mathcal{E}_{a_4}$}
\put(56,79.5){$\go P_t$}
\put(51.5,48){$\go P_c$}
\put(54.5,28){$\go P_{\infty}$}
\put(61,0.5){$\go P$}
\end{small}     
  \end{overpic}
}
\end{center} 
\caption{
For the example given in Eq.\ (\ref{par:ex}) the loci $\mathcal{E}_{a_4}$, $\mathcal{E}_c$ and $\mathcal{E}_{t}$ with 
$t=\tfrac{69}{20}$ are illustrated. 
(a) For $h_0=1$ the $h_1$- and $h_2$-parameter lines of the given rational quadratic parametrization are displayed.  
(b) The loci are sliced (along the not drawn axis of rotation $\go c$) in order to visualize their positioning with respect to the cubic $\go P$  
on which the points $\go P_{\infty}=\sigma(\go U)$, $\go P_c=\go C$ and $\go P_t$ are highlighted. 
Note that $\go P_{a_4}=\go W$ is the real ideal point of $\go P$.
}
  \label{fig-1}
\end{figure}

For $t\neq a_4$ all $\mathcal{E}_t$ are ellipsoids of rotation, which have the same  center point $\go C$ and axis of rotation $\go c$ 
(see Fig.\ \ref{fig-1}b). 
In detail, $\go C$ is the point of the straight cubic circle (\ref{Pt}) for the value $t=c$ with $c:=\tfrac{a_4^2-a_c^2-a_r^2}{2(a_4-a_r)}$ 
(for $a_4=a_r$ we get $c=\infty$ thus $\go p_{\infty}=\go U=\go m_5$ holds, which implies $\go C=\go M_5$) and 
$\go c$ is parallel to the $z$-axis of $\mathcal{F}_0^{\prime}$. 
Moreover the vertices on $\go c$ have distance $|a_4-t|$ from $\go C$ and the  squared radius of the equator circle 
 equals $(a_r-t)^2+a_c^2$.
Note that for $a_4\neq a_r$  the only sphere within the described set of ellipsoids is $\mathcal{E}_c$. 
For $a_4= a_r$ no such sphere exists.  

$\mathcal{E}_{a_4}$ is a circular disc in the Darboux plane $z=p_4$ (w.r.t.\ $\mathcal{F}_0^{\prime}$)
centered in $\go C$.

\begin{remark}
The existence of these ellipsoids was already known to Duporcq \cite[\S 9]{duporcq}, who 
used them to show that the spherical trajectories are algebraic curves of degree 4 (see Fig.\ \ref{fig-2}). \hfill $\diamond$
\end{remark}  

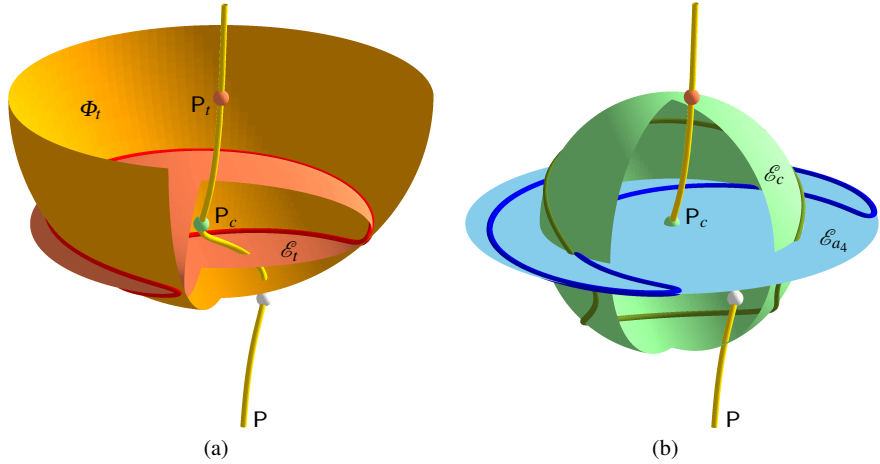
\begin{figure}[top]
\begin{center} 
\subfigure[]{ 
 \begin{overpic}
    [width=57mm]{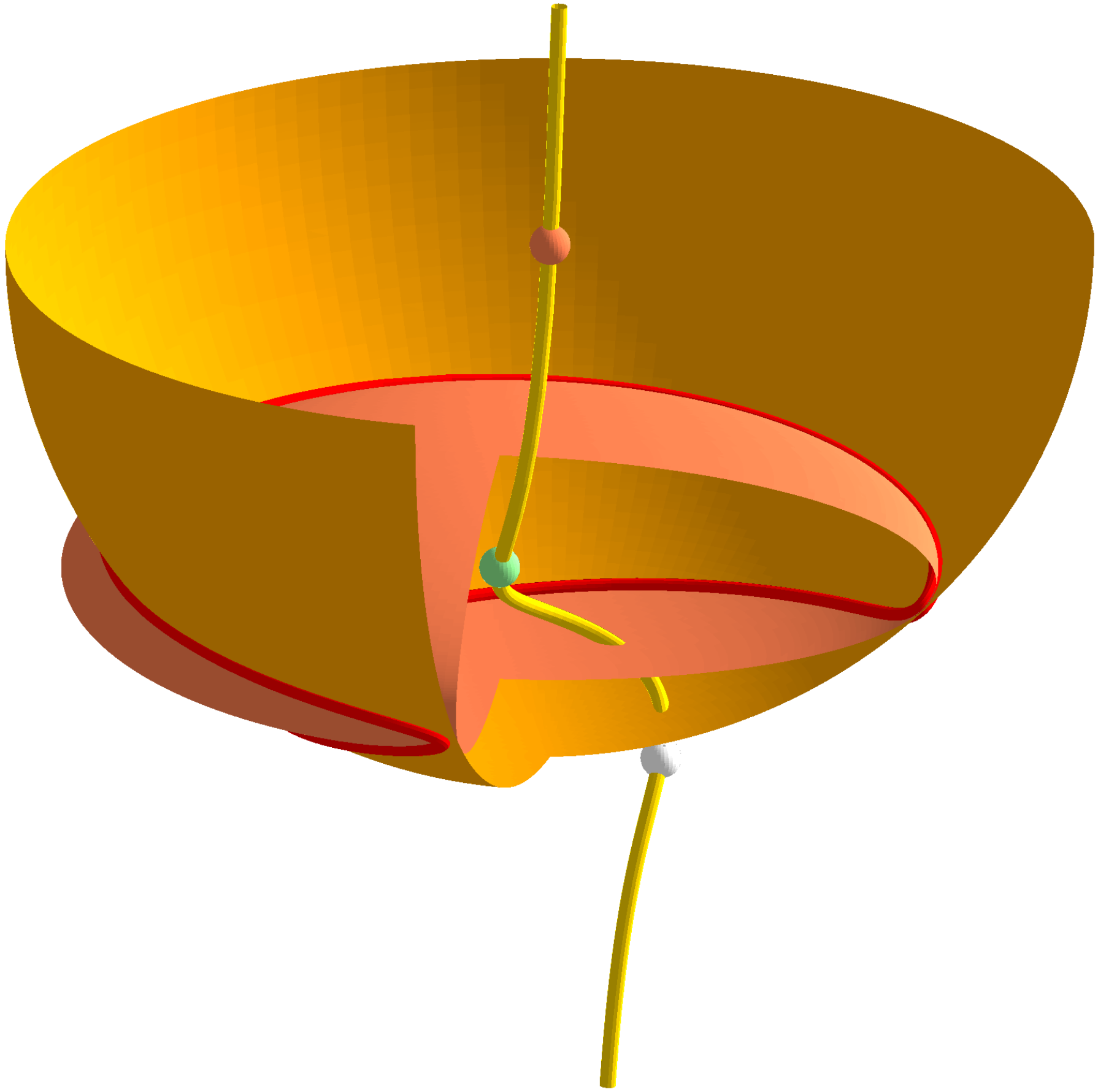}
\begin{small}
\put(64,40){$\mathcal{E}_{t}$}
\put(17,73){$\Phi_t$}
\put(43,74){$\go P_t$}
\put(48.5,48){$\go P_c$}
\put(57.5,0.5){$\go P$}
\end{small}         
  \end{overpic} 
}
\hfill
\subfigure[]{ 
\begin{overpic}
    [width=55mm]{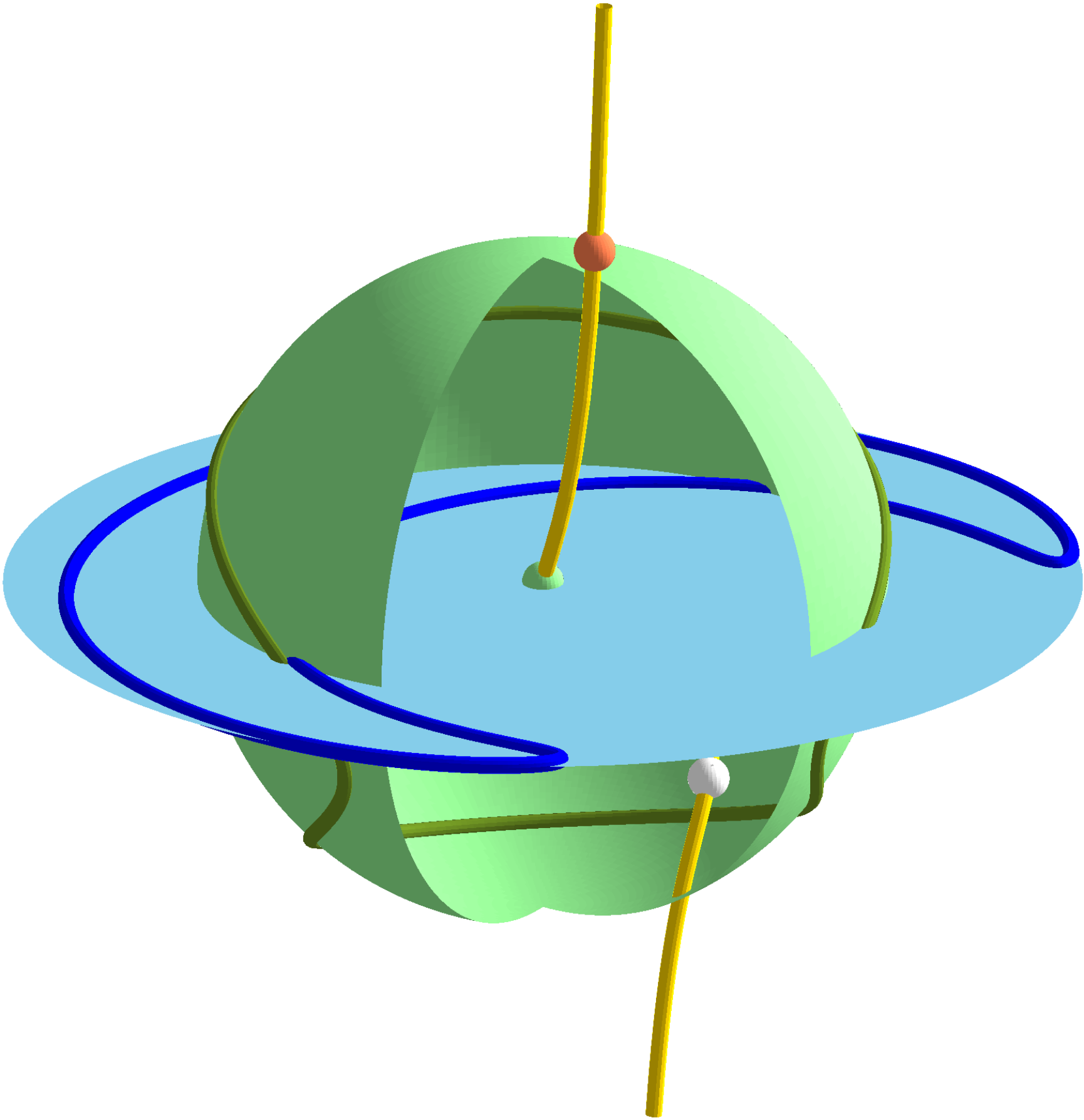}
\begin{small}
\put(70,58){$\mathcal{E}_{c}$}
\put(83,43){$\mathcal{E}_{a_4}$}
\put(51.5,48){$\go P_c$}
\put(61,0.5){$\go P$}
\end{small}         
  \end{overpic}
}
\end{center}
\caption{
For the example given in Eq.\ (\ref{par:ex}) and $p_5=6$ the trajectory of $\go p_t$ with $t=\tfrac{69}{20}$ is illustrated in (a) as 
the intersection curve of $\mathcal{E}_t$ and the sphere $\Phi_t$ centered in $\go P_t$.
The trajectories of the points $\go p_c$ and $\go p_{a_4}$, where this intersection procedure for their generation 
fails, are visualized in (b) as curves on  $\mathcal{E}_c$ and  $\mathcal{E}_{a_4}$, respectively. 
}
  \label{fig-2}
\end{figure}

Based on this geometric property, recovered by line-symmetric motions, 
we can formulate the condition for the self-motion to be real as follows: 
\begin{enumerate}[$\bullet$]
\item
$w\neq 0$:  We can reduce the problem  to a planar one by intersecting the plane spanned by $\go P_0=\go M_1$ and 
$\go c$ with $\mathcal{E}_0$ and the sphere with radius $R_1$ centered in $\go P_0$. 
Now there exists an interval $I_0=]I_-,I_+[$ such that for $R_1\in I_0$ the two resulting conics have at least two distinct real intersection points. 
It is well known  (e.g.\ \cite{conic}) that the computation of the limits $I_-$ and $I_+$ of the reality interval $I_0$ 
leads across an algebraic problem of degree 4 (explicitly solvable). 
Thus for a real self-motion we have to choose $R_1\in I_0$ and solve Eq.\ (\ref{zwei}) for $p_5$.
\item
$w=0$: Now $\go P_0$ coincides with $\go C$ and the interval collapses to the single value $R_1=|a_4|$, which 
can be seen from Eq.\ (\ref{zwei}). Moreover $p_5$ can be chosen arbitrarily.  
\end{enumerate}

These considerations also show that any pentapod of Type 1 and 2 has real self-motions if the leg-parameters are chosen properly. 
Note that this is e.g.\ not the case for some designs of Type 5 pentapods described in \cite[Section 6]{linear_penta}, where it 
was also proven that pentapods with self-motions have a quartically solvable  direct 
kinematics.  
It is possible to use this advantage (closed form solution) of pentapods with self-motions without any risk\footnote{
A self-motion is dangerous as it is uncontrollable and thus a hazard to man and machine.}, by 
designing linear pentapods of Type 1 and Type 2, which are guaranteed free of self-motions within their workspace. 

A sufficient condition for that is that (at least) for one of the five legs $\go p_t\go P_t$ 
of the pentapod the corresponding reality interval $I_t$ is disjoint with the interval of the maximal and minimal leg length 
implied by the mechanical realization. This condition for a self-motion free workspace gets especially simple if 
$\go p_c\go P_c$ is this leg.

\begin{figure}[top]
\begin{center} 
 \begin{overpic}
    [width=60mm]{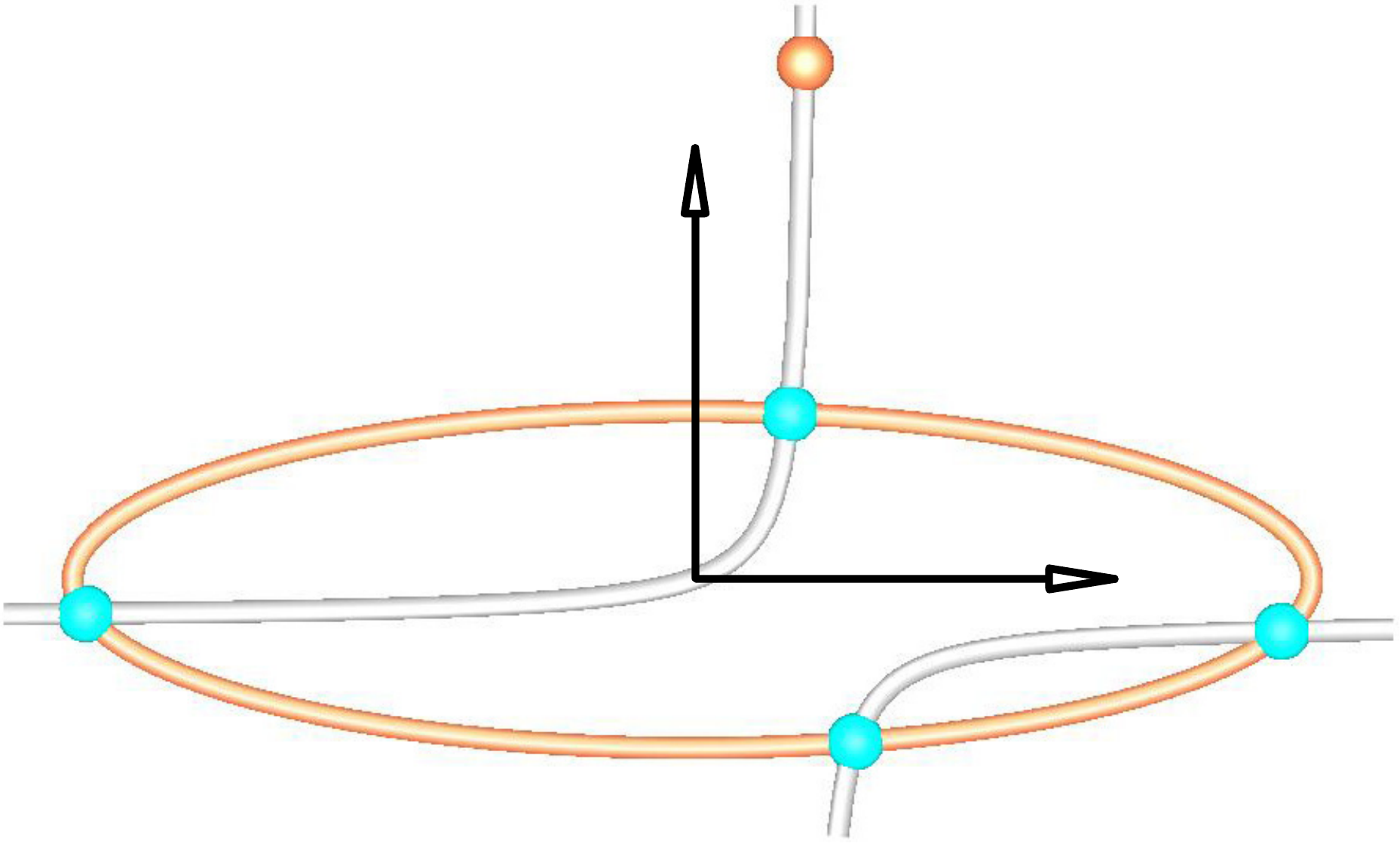}
\begin{small}
\put(60.5,53.5){$\go P_t$}
\put(59,33){$\go F_1$}
\put(1.5,10.5){$\go F_4$}
\put(93.5,9.5){$\go F_2$}
\put(62.5,1){$\go F_3$}
\put(59.5,43){$\ell$}
\put(81,18){$\xi$}
\put(48.2,52.5){$\zeta$}
\put(30,24){$\go k$}
\end{small}         
  \end{overpic} 
\end{center} 
\caption{Lagrange curve $\ell$ intersects the ellipse $\go k$ in the pedal points $\go F_i$ ($i=1,\ldots ,4$) 
of $\go k$ w.r.t.\ $\go P_t$.
}
  \label{fig-4}
\end{figure}
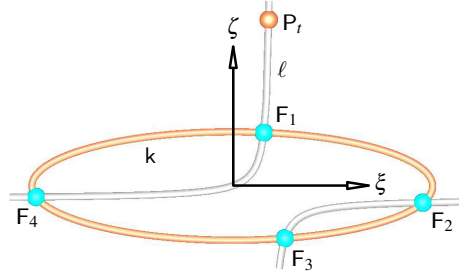     

\begin{example}
We only provide an example for the most general case; i.e.\ Type 1 pentapod with self-motion. 
The parameters are chosen as follows:
\begin{equation}\label{par:ex}
a_4=2, \quad
A=-1,\quad 
C=-5,\quad
a_r=7,\quad
a_c=4,
\end{equation}
with respect to the frames $\mathcal{F}_0^{\prime}$ and $\mathcal{F}^{\prime}$, respectively.
In Fig.\ \ref{fig-4} the planar intersection of $\mathcal{E}_t$ for $t=\tfrac{69}{20}$ 
with the plane spanned by $\go P_t$ and $\go c$ is illustrated. The half-axes lengths of 
this ellipse $\go k$ are 
$k_1=\tfrac{\sqrt{11441}}{20}$ and 
$k_2=\tfrac{29}{20}$.

With respect to a planar Cartesian frame $(\xi,\zeta)$ aligned with the axes of the ellipse $\go k$ the 
point $\go P_t$ has the coordinates $(\xi_t,\zeta_t):=(\tfrac{530}{469081}\sqrt{743665},\tfrac{5300}{1189})$. 
Then the so-called {\it Lagrange curve} $\ell$, which intersects $\go k$ in the pedal points with respect to $\go P_t$, 
has the following parametrization (parameter k) with respect to the planar Cartesian frame 
$(\xi,\zeta)$:
\begin{equation}
\left( k , \tfrac{\kappa_1\zeta_tk}{(\kappa_1-\kappa_2)k+\kappa_2\xi_t}\right)\quad\text{with}\quad
\kappa_1=\tfrac{1}{k_1^2} \quad\text{and}\quad \kappa_2=\tfrac{1}{k_2^2}.
\end{equation}
Plugging this parametrization into the equation of the ellipse $\go k$, given by 
$\kappa_1\xi^2+\kappa_2\zeta^2=1$ yields a quartic equation. 
Therefore the pedal points $F_i$ for $i=1,\ldots 4$ and the corresponding distances $l_i$ to $\go P_t$ 
can be computed explicitly. Then $I_-$ and $I_+$ of $I_t$ are given by the minimal and maximal value of 
$\left\{ l_1,l_2,l_3,l_4\right\}$. For the example under consideration, the corresponding rounded 
numerical values are 
$I_-\approx 3.02850$ and $I_+\approx 7.82039$, respectively. \hfill $\diamond$ 
\end{example}

\begin{remark}
A straight forward computation for the general case (cf.\ example given in the Appendix)  
shows that the basic surface is of degree 5 (see Fig.\ \ref{fig-5}a). 
Moreover the intersection of this basic surface with the ideal plane decomposes into 
a cubic curve and the two tangents $(0:1:\pm i:\RR)$ to the absolute circle. 
According to Krames \cite[Theorem 8 and the subsequent paragraph]{krames_fuss}
a general point of the moving system $\Sigma$ has a trajectory of degree 6 under the corresponding line-symmetric motion 
(see Fig.\ \ref{fig-6}).

Note that these quintic basic surfaces differ from those studied by Krames in \cite{krames_grad5}, 
as during the corresponding line-symmetric motions of the latter  no point of $\Sigma$ can have a spherical path due 
to the following reason: Krames showed in \cite[page 230]{krames_grad5} that every point has a trajectory of degree 4. 
Due to \cite[Theorem 7]{krames} the existence of points with spherical trajectories already implies that 
the basic surface has to be of degree 4; a contradiction. 
\hfill $\diamond$
\end{remark}

\begin{figure}[top]
\begin{center} 
\subfigure[]{
\begin{overpic}
    [width=41mm]{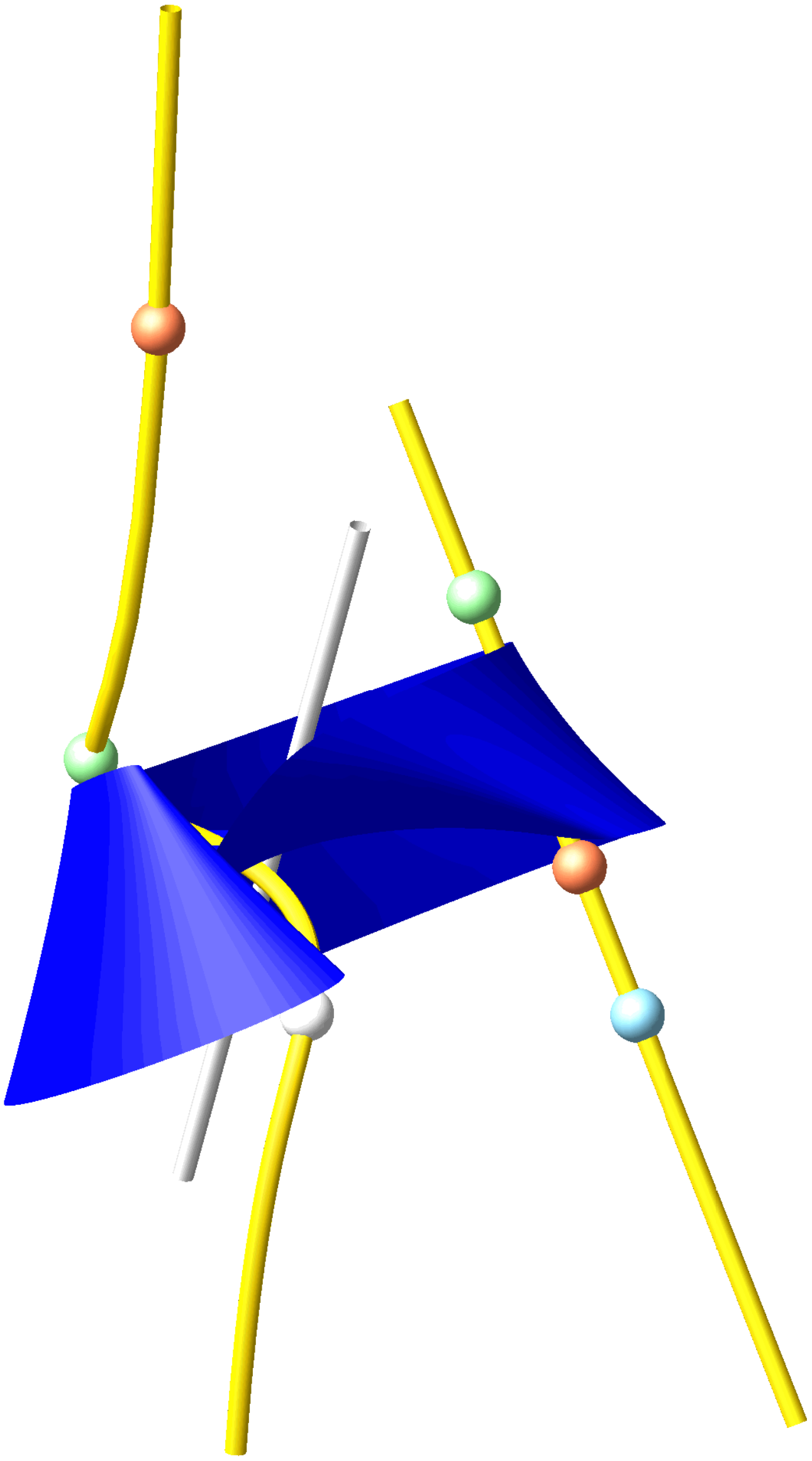}
\begin{small}
\put(35,59){$\overline{\go p}_c$}
\put(46.3,30){$\overline{\go p}_{a_4}$}
\put(42.8,39){$\overline{\go p}_{t}$}
\put(49,3.5){$\overline{\go p}$}
\put(9,20){$\go g$}
\put(23.5,28.5){$\go P_{\infty}$}
\put(7,96){$\go P$}
\put(5,75.5){$\go P_t$}
\put(0.5,49.5){$\go P_c$}
\end{small}             
  \end{overpic}  
}
\hfill
\subfigure[]{ 
 \begin{overpic}
    [width=63mm]{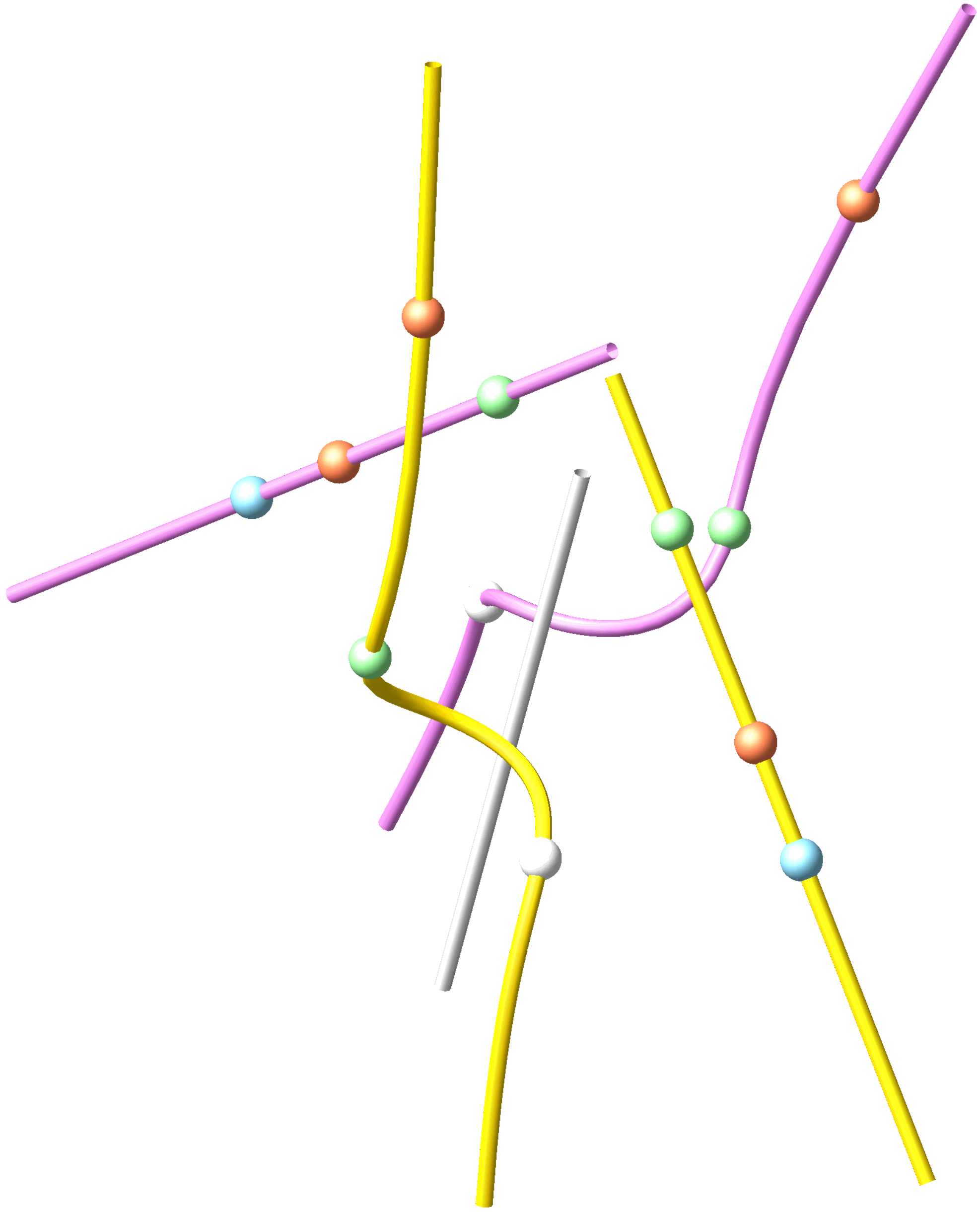}
\begin{small}
\put(50.3,54){$\overline{\go p}_c$}
\put(68.5,29){$\overline{\go p}_{a_4}$}
\put(64,2.2){$\overline{\go p}\in \Sigma_0$}
\put(33,19){$\go g$}
\put(37.5,92){$\go P\in\Sigma_0$}
\put(28.9,72.5){$\go P_t$}
\put(23.9,45){$\go P_c$}
\put(62.5,55){$\overline{\go P}_c$}
\put(73,81.5){$\overline{\go P}_t$}
\put(67,96){$\overline{\go P}\in \Sigma$}
\put(37,53){$\overline{\go P}_{\infty}$}
\put(47,28.5){$\go P_{\infty}$}
\put(65,38.5){$\overline{\go p}_t$}
\put(39.5,63){$\go p_c$}
\put(26,58){$\go p_t$}
\put(18.5,55){$\go p_{a_4}$}
\put(1,48){$\go p\in \Sigma$}
\end{small}             
  \end{overpic} 
 }
\end{center} 
\caption{(a) 
For the example given in Eq.\ (\ref{par:ex}) the basic surface is illustrated for the 
parameter values highlighted in Fig.\ \ref{fig-3}. In addition $\go P$ and $\overline{\go p}$ are visualized, 
where the latter denotes the pose of $\go p$ such that its half-turns about the generators of the 
basic surface yield the self-motion. 
(b) The construction outlined by Krames \cite[page 416]{krames} is illustrated with 
respect to the generator $\go g$ of the basic surface: 
As $\go P_{a_4}$ (resp.\ $\overline{\go p}_{\infty}$) is the real ideal point of $\go P$ 
(resp.\ $\overline{\go p}$),  the trajectory of $\go p_{a_4}$ (resp.\ $\overline{\go P}_{\infty}$) 
under the self-motion $\mu$ is planar. 
The (Mannheim) plane $\in\Sigma$, which contains the point  $\go P_{\infty}$ 
(resp.\  $\overline{\go p}_{a_4}$) and is orthogonal to the direction of the real ideal point 
$\go p_{\infty}$ (resp.\ $\overline{\go P}_{a_4}$) of $\go p$ (resp.\ $\overline{\go P}$) 
in the displayed pose, slides through the point $\go P_{\infty}$ (resp.\  $\overline{\go p}_{a_4}$)
during the complete motion $\mu$. 
}
  \label{fig-5}
\end{figure}
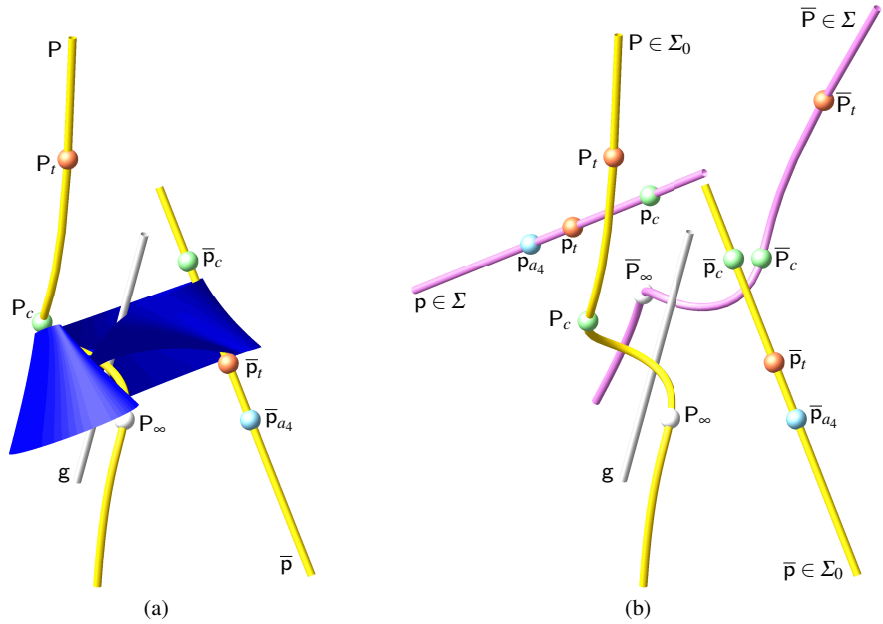

\section{Conclusion and open problem}\label{new:solutions}

Krames \cite[page 416]{krames} outlined the following construction illustrated in Fig.\ \ref{fig-5}b: Assume that $\go p$ is in an arbitrary pose of the 
self-motion $\mu$ with respect to $\go P$, where $\go g$ denotes the generator of the basic surface, which corresponds to this pose. 
Moreover $\overline{\go p}$ and $\overline{\go P}$ are obtained by the reflexion of  $\go p$ and  $\go P$, respectively, with respect to $\go g$, 
where $\overline{\go p}$ belongs to the fixed system $\Sigma_0$ and $\overline{\go P}$ to the moving system $\Sigma$. 
Then under the self-motion $\mu$ also the points of $\overline{\go P}$ are located on spheres with centers on the line $\overline{\go p}$. 

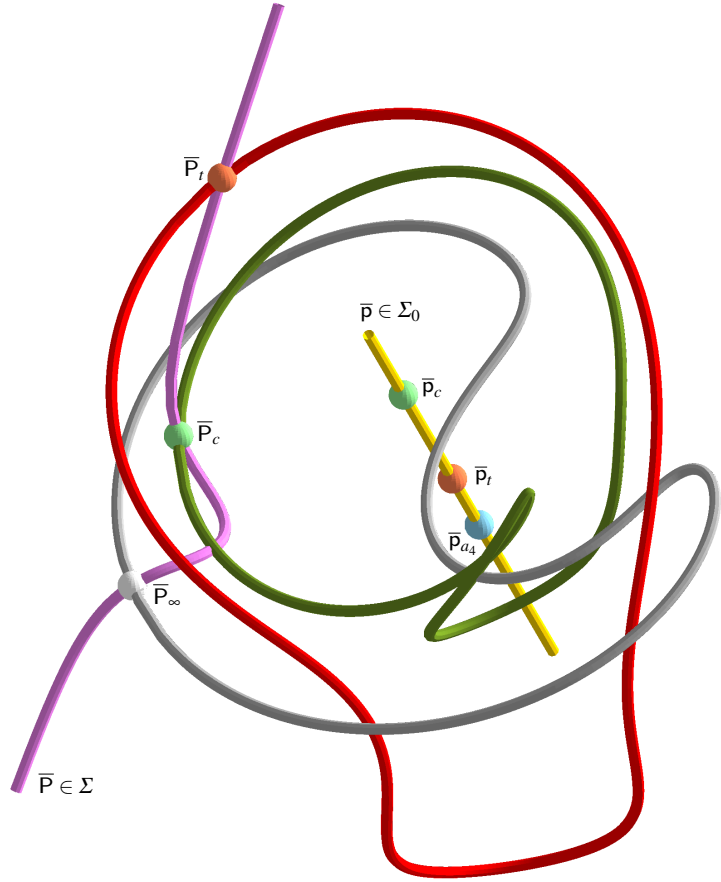
\begin{figure}[top]
\begin{center} 
\begin{overpic}
    [width=95mm]{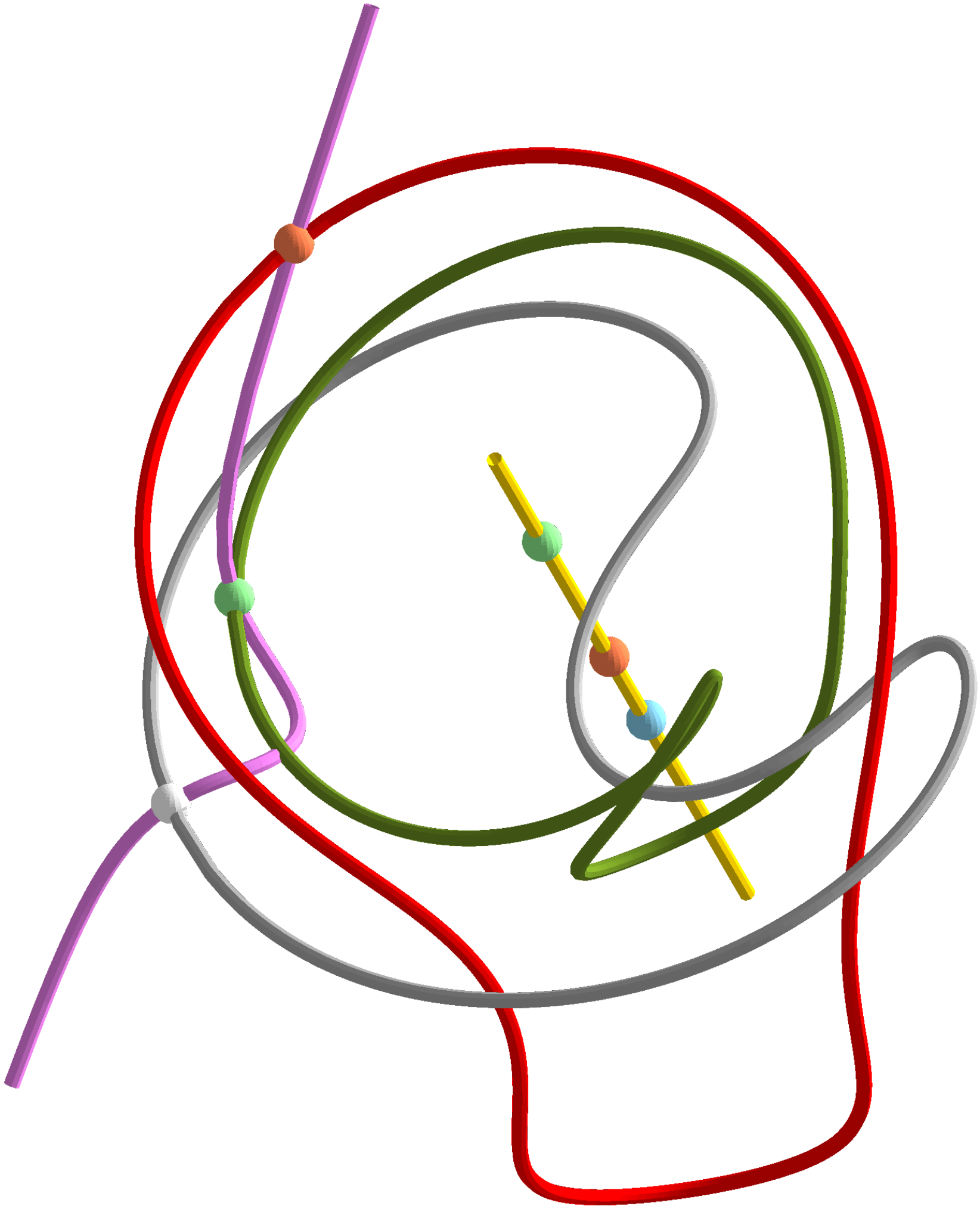}
\begin{small}
\put(40,64){$\overline{\go p}\in\Sigma_0$}
\put(47,55.5){$\overline{\go p}_c$}
\put(53,45.5){$\overline{\go p}_t$}
\put(50,38){$\overline{\go p}_{a_4}$}
\put(3.5,10){$\overline{\go P}\in\Sigma$}
\put(16.5,31.5){$\overline{\go P}_{\infty}$}
\put(21.5,50){$\overline{\go P}_c$}
\put(20,80){$\overline{\go P}_t$}
\end{small}             
  \end{overpic}  
\end{center} 
\caption{
Also the points of the cubic $\overline{\go P}\in\Sigma$ have trajectories of degree 6 under the corresponding line-symmetric motion. 
The paths of $\overline{\go P}_{\infty}$ (gray), $\overline{\go P}_t$ (red) and $\overline{\go P}_c$ (green) are illustrated for the 
line-symmetric motion implied by the basic surface displayed in Fig.\ \ref{fig-5}a. 
Note that in contrast to the Figs.\ \ref{fig-3}, \ref{fig-1}, \ref{fig-2} and \ref{fig-5}, 
the camera location was changed as otherwise the  
planar trajectory of $\overline{\go P}_{\infty}$ would have been collapsed into a line-segment.}
\label{fig-6}
\end{figure}

We can apply this construction for each line-symmetric motion of Theorem \ref{main},
which yields new solutions for the Borel Bricard problem, with the exception of one special case 
where $\go W\in\overline{\go p}$ holds (i.e.\ $h_1=h_2=0$ or $h_0=0$), which was 
already  given by Borel in \cite[Case Fa4]{borel}. 
Moreover for this case Borel noted that beside $\go p$ and $\overline{\go P}$ only two
imaginary planar cubic curves, which are located in the isotropic planes through $\go p$,  
run on spheres. The example given in the Appendix shows that this also 
holds true for the general case. 

Thus the problem remains to determine all line-symmetric motions of Theorem \ref{main} 
where additional real points (beside those of $\go p$ and $\overline{\go P}$) run on spheres. 
Until now the only examples with this property, which are known to the author, are the BB-II motions (cf.\ Section \ref{zurueck}).

\begin{acknowledgement}
This research is funded by Grant No.~P~24927-N25 of the Austrian Science Fund FWF within the
project "Stewart Gough platforms with self-motions".
\end{acknowledgement}

\newpage

\section*{Appendix}

We continue the example given in Eq.\ (\ref{par:ex}) for the parameters
\begin{equation}
h_0=1,\quad
h_1=\tfrac{3}{2},\quad
h_2=\tfrac{1}{2}.
\end{equation}
Now $H=0$ implies the condition $p_5=\tfrac{527538}{82369}$.

\subsubsection*{Degree of basic surface}

As already mentioned at the end of Section \ref{blsm}, 
$(e_1:e_2:e_3:f_1:f_2:f_3)$ are the Pl\"ucker coordinates of the generators of the basic surface 
with respect to the fixed frame $\mathcal{F}_0^{\prime}$ if $e_0=f_0=0$ holds. 
Each generator $\go g$ can now be parametrized by
\begin{equation}\label{eq:gen}
\go g:\quad \Vkt G + \gamma\, \Vkt e \quad \text{with} \quad \Vkt G:=\Vkt f\kreuz\Vkt e,
\end{equation}
where $\Vkt e:=(e_1,e_2,e_3)^T$ and $\Vkt f:=(f_1,f_2,f_3)^T$. Note that $\Vkt G$ is the coordinate vector (with respect to $\mathcal{F}_0^{\prime}$) 
of the pedal point of $\go g$ with respect to the origin of $\mathcal{F}_0^{\prime}$. 
As $f_0,\ldots ,f_3$ are computed as in Theorem \ref{main}, the parametrization of points of the basic surface, which have 
fixed coordinates $(X,Y,Z)^T$, 
only depends on $e_1$, $e_2$, $e_3$ and $\gamma$. Moreover the condition $F=0$ (implied by $f_0=0$) has to hold with:
\begin{equation} \label{yo}
\begin{split}
F:=&1561e_1^3
-1708e_2^3
-2870e_3^3
-1708e_1^2e_2
+2173e_1^2e_3 \\
&+1561e_1e_2^2
+2173e_2^2e_3
-8525e_1e_3^2 
-5070e_2e_3^2.
\end{split}
\end{equation}
From $F=0$ and the three equations $(X,Y,Z)^T=\Vkt G + \gamma\, \Vkt e$ the parameters 
$e_1$, $e_2$, $e_3$ and $\gamma$ can be eliminated (by Gr\"obner basis elimination techniques) 
and we finally end up with the following implicit representation of the basic surface 
with respect to $\mathcal{F}_0^{\prime}$:
 
\begin{small}
\begin{equation*}
\begin{split}
&46930000-59658690X^2YZ^2-187188000X-211012100Z-214586000Y \\
&-100313675XY^2Z^2+51139382X^2Y^2Z+199127898X^2YZ-109272380XYZ^2 \\
&+507256879XY^2Z+812997056XYZ+18368287X^5+156741893X^4+193802233X^3 \\
&-138195885X^2+323499460Y^2-423262296Y^3-191807553Y^4-20098036Y^5 \\
&+115381175Z^2-53539850Z^3-20098036X^4Y+36736574X^3Y^2-40196072X^2Y^3 \\
&+18368287XY^4-5063828X^3Y-35065660X^2Y^2-5063828XY^3+378103208X^2Y \\
&-582502914XY^2+195340960XY+25569691X^4Z-100313675X^3Z^2-33771290X^2Z^3 \\
&+507256879X^3Z-176781955X^2Z^2-84016380XZ^3+1163441469X^2Z+34883415XZ^2 \\
&+556424125XZ+13179040YZ^3+639976950YZ-33771290Y^2Z^3-28318339Y^2Z \\
&+199127898Y^3Z-59658690Y^3Z^2+25569691Y^4Z+96060335Y^2Z^2-41015170YZ^2=0.
\end{split}
\end{equation*} 
\end{small}
It can easily be seen that this is a polynomial equation of degree 5. 
If we intersect this quintic surface with the ideal plane (by homogenizing it 
with the homogenization variable $W$ and then setting $W=0$) the intersection curve splits 
up into the two tangents to the absolute circle $X\pm Yi=0$ and the cubic curve
given by the direction vectors of the basic surface's generators. 
Therefore this cubic curve in the ideal plane is given by $F=0$ (with $F$ of Eq.\ (\ref{yo})) if one substitutes 
$e_1$ by $X$, $e_2$ by $Y$ and $e_3$ by $Z$, respectively. 

As the ideal curve of the basic surface contains two tangents of the absolute circle,  
a general point of the moving system $\Sigma$ has a trajectory of degree 6 
under the corresponding line-symmetric motion according to Krames \cite[Theorem 8 and the subsequent paragraph]{krames_fuss}. 

\begin{remark}
It should be mentioned that all basic surfaces and 
trajectories can be parametrized for the general case due to Remark \ref{rem:par}, 
which was used for the generation of Figs.\ \ref{fig-5}a and 
\ref{fig-6}, respectively. \hfill $\diamond$
\end{remark}

\subsubsection*{Imaginary planar cubic curves}

Now we determine the set of all points $\in\Sigma$, which are running on spheres during the above given 
line-symmetric motion.

By using the Study parametrization of Euclidean displacements, the condition that a point of $\Sigma$ with coordinates $(x,y,z)^T$ 
(with respect to the moving frame $\mathcal{F}^{\prime}$) is located 
on a sphere centered in a point with coordinates $(X,Y,Z)^T$ (with respect to $\mathcal{F}_0^{\prime}$)
is a quadratic homogeneous equation according to Husty \cite{manfred}. 
For $e_0=0$ this so-called sphere condition $\Lambda=0$ reads as: 
\begin{equation*}
\begin{split}
\Lambda:=\,\, 
&(x^2+y^2+z^2+X^2+Y^2+Z^2-R^2)(e_1^2+e_2^2+e_3^2) +4(f_0^2+f_1^2+f_2^2+f_3^2) \\
&-2(xX-yY-zZ)e_1^2 
+2(xX-yY+zZ)e_2^2
+2(xX+yY-zZ)e_3^2 \\
&-4(yX+xY)e_1e_2
-4(zX+xZ)e_1e_3 
-4(zY+yZ)e_2e_3 \\
&-4(x+X)(e_3f_2-e_2f_3)
-4(y+Y)(e_1f_3-e_3f_1) 
-4(z+Z)(e_2f_1-e_1f_2)\\
&+4(x-X)e_1f_0 
+4(y-Y)e_2f_0 
+4(z-Z)e_3f_0, 
\end{split}
\end{equation*}
where $R$ denotes the radius of the sphere.  

The corresponding Maple Worksheet of the following computations can be downloaded as {\tt mws} file and {\tt pdf} file 
from the links provided in the footnote.\footnote{The links are {\tt http://www.dmg.tuwien.ac.at/nawratil/linearpentapod.mws} and 
{\tt http://www.dmg.tuwien.ac.at/nawratil/linearpentapod.pdf}, respectively.}  
We compute $f_0,\ldots ,f_3$ as in Theorem \ref{main} and plug them into $\Lambda$. 
The numerator of the resulting expression is denoted by $\Gamma[161]$, which is of degree $4$ 
in $e_1,e_2,e_3$. Therefore we can make an ansatz of the form:
\begin{equation}
F(\eta_1e_1+\eta_2e_2+\eta_3e_3)-\Gamma=0.
\end{equation}
As this equation has to be fulfilled identically for all $e_1,e_2,e_3$, the coefficients 
with respect to these variables have to vanish. This results in a set of 15 equations from which we  
eliminate (by Gr\"obner basis elimination techniques) the unknowns $\eta_1,\eta_2,\eta_3,X,Y,Z,R$. 
We finally end up with an ideal $I$ in  $x,y,z$ of degree 10 and dimension 1; i.e.\ there exists a curve of 
degree 10, those points run on spheres during the line-symmetric motion. 

This curve of degree 10 has to split in at least three components as we already know of the existence of a linear and a cubic component, namely $\go p$ and 
$\overline{\go P}$, respectively. Now we show that the remaining sextic splits up into conjugate complex planar 
cubics, which are located in the isotropic planes $\varepsilon_1,\varepsilon_2\in\Sigma$ through $\go p$. These isotropic planes are given by:
\begin{equation} 
\varepsilon_{1,2}:\quad
91x-84y-126z-122\pm(147y-98z+714)i=0, 
\end{equation} 
with respect to $\mathcal{F}^{\prime}$.
We express $x$ from it and plug it into the generating polynomials of $I$. It can easily be checked that 
the greatest common divisor of all resulting expression equals $(102+21y-14z)(\Re\pm \Im i)$ with
\begin{equation}
\begin{split}
\Re:=&
274400y^3+3374238yz+13169366z+3927840y^2-30870y^2z \\
&-5472908z^2-1165514yz^2+15910300y+113190z^3+17761620,
\end{split}
\end{equation}
\begin{equation}
\begin{split}
\Im:=
&984410y^2z-1840195y^2+9573816yz-115248yz^2-15809850y \\
&-29479660+817369z^2+20061237z-408170z^3.
\end{split}
\end{equation}
The condition $102+21y-14z=0$ yields the line $\go p$ and 
$\Re\pm \Im i=0$ determines the imaginary planar cubic curves. 
This result closes the Appendix.

\begin{thebibliography}{99.}

\bibitem {linear_penta}
Nawratil, G., Schicho, J.: Self-motions of pentapods with linear platform. Robotica, accepted [arXiv:1407.6126]

\bibitem {borel}
Borel, E.:
M\'{e}moire sur les d\'{e}placements \`{a} trajectoires sph\'{e}riques. 
M\'{e}moire pr\'{e}sente\'{e}s par divers savants \`{a} l'Acad\'{e}mie des Sciences 
de l'Institut National de France \textbf{33}(1) 1--128 (1908)

\bibitem {bricard}  
Bricard, R.:
M\'{e}moire sur les d\'{e}placements \`{a} trajectoires sph\'{e}riques. 
Journal de \'{E}cole Polytechnique(2) \textbf{11} 1--96 (1906)

\bibitem {husty_bb} 
Husty, M.:
E.\ Borel's and R.\ Bricard's Papers on Displacements with Spherical Paths and their Relevance to Self-Motions of Parallel Manipulators.
International Symposium on History of Machines and Mechanisms (M.\ Ceccarelli ed.), 163--172, Kluwer (2000)

\bibitem{krames}
Krames, J.:
Zur Bricardschen Bewegung, deren s\"amtliche Bahnkurven auf Kugeln liegen. 
Monatsheft f\"ur Mathematik und Physik \textbf{45} 407--417 (1937)

\bibitem{koenigs}
Koenigs, G.:
Le\c{c}ons de Cin\'{e}matique (avec notes par G.\ Darboux). Paris (1897)

\bibitem{mannheim}
Mannheim, A.:
Principes et D\'{e}veloppements de G\'{e}om\'{e}trie Cin\'{e}matique. Paris (1894)

\bibitem{duporcq}
Duporcq, E.:
Sur le d\'{e}placement le plus g\'{e}n\'{e}ral d'une droite dont tous les points d\'{e}crivent des trajectoires sph\'{e}riques.
Journal de math\'{e}matiques pures et appliqu\'{e}es (5) \textbf{4} 121--136 (1898)

\bibitem{bottema}  
Bottema, O., Roth, B.:
Theoretical Kinematics.
North-Holland Publishing Company (1979)

\bibitem{hartmann}
Hartmann, D.:
Singular Stewart-Gough Platforms.
Master Thesis, Department of Mechanical Engineering, McGill University, Montreal, Canada (1995)

\bibitem{krames2}
Krames, J.:	
Die Borel-Bricard-Bewegung mit punktweise gekoppelten orthogonalen Hyperboloiden. 
Monatsheft f\"ur Mathematik und Physik \textbf{46} 172--195 (1937)

\bibitem{duporcq_n}
Nawratil, G.: 
Correcting Duporcq's theorem. 
Mechanism and Machine Theory \textbf{73} 282-–295 (2014) 

\bibitem{conic}
Chernova, N., Wijewickrema, S.:
Algorithms for projecting points onto conics.
Journal of Computational and Applied Mathematics \textbf{251} 8--21 (2013)

\bibitem{krames_fuss}
Krames, J.:	
\"Uber Fu{\ss}punktkurven von Regelfl\"achen und eine besondere Klasse von Raumbewegungen. 
Monatshefte f\"ur Mathematik und Physik \textbf{45} 394--406 (1936)

\bibitem{krames_grad5}
Krames, J.:	
\"Uber eine konoidale Regelfl\"ache f\"unften Grades und die darauf gegr\"undete symmetrische Schrotung. 
Sitzungsbericht der \"Osterreichischen Akademie der Wissenschaften Mahem.-naturw.\ Klasse, Abteilung II,
\textbf{190}(4-7) 221--230 (1981)

\bibitem{manfred}
Husty, M.L.: 
An algorithm for solving the direct kinematics of general Stewart-Gough platforms. 
Mechanism and Machine Theory \textbf{31}(4) 365--380 (1996)

\end{thebibliography}
\end{document}